\documentclass[twoside]{article}

\usepackage[accepted]{aistats2025}
\usepackage[round]{natbib}

\usepackage{amsthm}
\usepackage{mathtools}
\usepackage{amssymb} 
\usepackage[colorlinks=true,linkcolor=blue,citecolor=blue]{hyperref}
\usepackage{thmtools}
\usepackage{thm-restate}
\usepackage{cleveref}
\usepackage{float} 
\usepackage{soul}
\usepackage{enumitem}
\usepackage{caption}
\usepackage{amsmath,nccmath}
\usepackage{booktabs}
\usepackage{cuted}
\usepackage{algorithm}
\usepackage{algpseudocode}

\newtheorem{lemma}{Lemma}

\newtheorem{proposition}{Proposition}

\theoremstyle{definition}
\newtheorem{remark}{Remark}

\usepackage{bbm}
\allowdisplaybreaks

\usepackage[detect-all]{siunitx}  
\newcommand{\myrightleftarrows}[1]{\mathrel{\substack{\xrightarrow{#1} \\[-.9ex] \xleftarrow{#1}}}}

\newcounter{relctr} 
\everydisplay\expandafter{\the\everydisplay\setcounter{relctr}{0}} 

\AtBeginDocument{} 

\begin{document}

%

\twocolumn[

\aistatstitle{Control, Transport and Sampling: Towards Better Loss Design}

\aistatsauthor{ Qijia Jiang \And David Nabergoj }

\aistatsaddress{ UC Davis \And  University of Ljubljana } ]

\begin{abstract}
Leveraging connections between diffusion-based sampling, optimal transport, and stochastic optimal control through their shared links to the Schr\"odinger bridge problem, we propose novel objective functions that can be used to transport $\nu$ to $\mu$, consequently sample from the target $\mu$, via \emph{optimally controlled} dynamics. We highlight the importance of the pathwise perspective and the role various optimality conditions on the path measure can play for the design of valid training losses, the careful choice of which offer numerical advantages in implementation. Basing the formalism on Schr\"odinger bridge comes with the additional practical capability of baking in inductive bias when it comes to Neural Network training.
\end{abstract}

\section{INTRODUCTION}
Traditionally, the task of sampling from un-normalized densities is largely delegated to MCMC methods. However, modern machine learning developments in optimal transport and generative modeling have greatly expanded the toolbox we have available for performing such tasks, which can further benefit from powerful advancements in deep learning. By reducing the problem to performing empirical risk minimization with neural networks, they hold promise especially for sampling from high-dimensional and multimodal distributions compared to MCMC-based alternatives. In this work, we propose novel training objectives that are amenable to tractable approximations for sampling (without access to data from the target, as typically considered in the MCMC setup), and demonstrate their numerical advantages compared to several related methods that are also built around forward-backward SDEs through time-reversals. Before elaborating on these connections and synthesis in Section~\ref{sec:framework}, we briefly summarize our contributions below.   
\begin{itemize}
\item We present a transport/control based sampler using Schr\"odinger bridge (SB) idea (generalizable to extended state space and nonlinear prior), which compared to MCMC-based approaches has two benefits: (1) alleviate metastability without being confined to local moves in the state space; (2) based on interpolating path, it comes with the ability to provide low-variance, unbiased normalizing constant estimates using the trajectory information (c.f. Section~\ref{sec:implement}).  
\item Unlike previously studied Schr\"odinger-bridge-based \citep{chen2021likelihood} and diffusion-based samplers \citep{berner2022optimal, vargas2022denoising, richter2023improved}, we explicitly enforce \emph{optimality} and \emph{uniqueness} of the trajectory, which also ensure that the dynamics reaches target in \emph{finite} time. Compared to IPF-based approaches for solving SB \citep{vargas2023transport}, our joint training approach of the forward and backward controls is not susceptible to prior forgetting, at the same time amenable to importance sampling correction for bias coming from neural network training (e.g., approximation error). The algorithm and setup we consider deviates from the more well-studied generative model use case of SB \citep{de2021diffusion, peluchetti2023diffusion, shi2024diffusion}. 
\item In contrast to previously proposed losses that follow optimal trajectory for sampling from the target \citep{vargas2023transport, liu2022deepRL}, our (discretized) training objectives have the numerical advantage of (1) vanishing variance of the stochastic gradient at the optimal control; (2) expense the need for evaluating expensive Laplacian terms. These heavily rely on the path measure based representations of our losses and estimators and generalize much beyond the half bridge case considered in \citep{zhang2021path}. Comparisons of the performance are given in Section~\ref{sec:experiment} w.r.t alternative loss proposals.
\item From a practical standpoint, our SB-based pathwise sampler is both gradient free, and comes with structure in its solution such that special NN architecture can be exploited for training. While we focus our attention on using the SB formalism for sampling, the same methodology can be used for e.g., optimal transport applications.
\end{itemize}

\section{GENERAL FRAMEWORK}
\label{sec:framework}
In this section we put several recently proposed (and diffusion-related) methods in context, rendering them as special instantiations of a more unifying \emph{path-wise} picture, which will in turn motivate the need for a control-based approach that involves designing an optimal path exactly interpolating between two distributions. We loosely follow the framework put forth in \citep{vargas2023transport} for parts of the exposition. 

\subsection{Setup}
We are interested in sampling from $p_{\text{target}}(x)=\mu(x)/Z$ by minimizing certain tractable loss, assuming sampling from $p_{\text{prior}}(z)=\nu(z)$ is easy, but unlike in generative modeling, even though analytical expression for $\mu$ is readily available, we do not have access to data from it that can be used to learn the score function. In this sense, in terms of ``transport mapping", the two sides are crucially not symmetric. 
%
%
Below we introduce forward-backward SDEs for formalizing such transitions. Pictorially, given a base drift $f$, we have the ``sampling process" (the two processes are time reversals of each other): 
\begin{equation}
\label{eqn:sample_pic}
\nu(z)\mathrel{\mathop{\myrightleftarrows{\rule{1.5cm}{0cm}}}^{\mathrm{\mathbb{P}^{\nu,f+\sigma u}}}_{\mathrm{\mathbb{P}^{\mu,f+\sigma v}}}} \mu(x)
\end{equation}
for tunable controls $u,v$ and terminal marginals $\nu,\mu$. Written as SDEs, they become (the drifts $u, v$ here are not independent)
\begin{align}
\label{eqn:forward}
dX_t&=[f_t(X_t) +\sigma u_t(X_t)]dt+\sigma \overrightarrow{dW_t},\; X_0\sim \nu\\
&\Rightarrow (X_t)_{t\in[0,T]} \sim \nonumber \overrightarrow{\mathbb{P}}^{\nu,f+\sigma u}\, ,\\
\label{eqn:backward}
dX_t&=[f_t(X_t)+\sigma v_t(X_t)]dt+\sigma \overleftarrow{dW_t},\; X_T\sim \mu\\
&\Rightarrow (X_t)_{t\in[0,T]} \sim  \overleftarrow{\mathbb{P}}^{\mu,f+\sigma v}\, . \nonumber
\end{align}
Operationally, \eqref{eqn:forward}-\eqref{eqn:backward} denote (picking $f=0$, and $\{z_t\}$ a sequence of i.i.d standard Gaussians for illustration)
\begin{align*}
X_t&=X_0+\int_0^t \sigma u_s(X_s)ds+\int_0^t \sigma \overrightarrow{dW_s}\;\;\\
&\Rightarrow X_{t+h}\approx X_t+h\sigma u_t(X_t)+\sqrt{h}\sigma z_t,\quad X_0\sim \nu \\
X_t&=X_T-\int_t^T \sigma v_s(X_s)ds-\int_t^T \sigma \overleftarrow{dW_s}\;\;\\
&\Rightarrow X_{t-h}\approx X_t-h\sigma v_t(X_t)+\sqrt{h}\sigma z_t,\quad X_T\sim \mu 
\end{align*}
where the forward (the usual It\^o's) and the backward integrals indicate different endpoints at which we make the approximation. For processes $(Y_t)_t,(Z_t)_t$, 
\begin{align}
\label{eqn:approx}
&\int_0^T a_t(Y_t)\, \overrightarrow{dZ_t} \approx \sum_i a_{t_i}(Y_{t_i})(Z_{t_{i+1}}-Z_{t_i})\, ,\\
\label{eqn:approx-1}
&\int_0^T a_t(Y_t)\, \overleftarrow{dZ_t} \approx \sum_i a_{t_{i+1}}(Y_{t_{i+1}})(Z_{t_{i+1}}-Z_{t_i})\, ,
\end{align}
which in particular implies the martingale property
\begin{align}
\label{eqn:martingale}
&\mathbb{E}_{X\sim \overrightarrow{\mathbb{P}}^{\nu,f+\sigma u}}\left[\int_0^t a_s(X_s)\, \overrightarrow{dW_s}\right]=0\, ,\\
&\mathbb{E}_{X\sim \overleftarrow{\mathbb{P}}^{\mu,f+\sigma v}}\left[\int_{T-t}^T a_s(X_s)\, \overleftarrow{dW_s}\right]=0\,. \label{eqn:martingale-1}
\end{align}
Forward/backward integral can be converted through 
\small\begin{equation}
\label{eqn:conversion}
\hspace{-0.2cm}\sigma^2\int_0^T (\nabla\cdot a_t)(Y_t)\,dt+\int_0^T a_t(Y_t) \,\overrightarrow{dZ_t}=\int_0^T a_t(Y_t)\,\overleftarrow{dZ_t}\, ,
\end{equation}\normalsize
which will be used repeatedly throughout. 
\begin{remark}[Nelson's identity] 
\label{rem:nelson}
The following relationship between drifts for the SDE \eqref{eqn:forward}-\eqref{eqn:backward} is well known: $\overrightarrow{\mathbb{P}}^{\nu,f+\sigma u}=\overleftarrow{\mathbb{P}}^{\mu,f+\sigma v}$ iff $\overrightarrow{\mathbb{P}}^{\nu,f+\sigma u}_T = \mu$, and $\sigma v_t(X_t)=\sigma u_t(X_t)-\sigma^2\nabla \log(\overrightarrow{\mathbb{P}}^{\nu,f+\sigma u}_t)=\sigma u_t(X_t)-\sigma^2\nabla \log(\overleftarrow{\mathbb{P}}^{\mu,f+\sigma v}_t)$ for all $t\in[0,T]$. This is used to give a more transparent derivation of the likelihood ratio in Lemma~\ref{lem:RN}.
\end{remark}

\subsection{Goal and Related Approaches}
\label{sec:related}
\paragraph{Forward KL} We would like the two path measures (with the specified two end-point marginals) to agree progressing in either direction. The methods in \citep{vargas2022denoising,zhang2021path} propose to set up a reference process with a similar structure as \eqref{eqn:sample_pic} with $f=0$:
\begin{equation}
\label{eqn:ref_pic}
\nu(z)\mathrel{\mathop{\myrightleftarrows{\rule{1.5cm}{0cm}}}^{\mathrm{\mathbb{P}^{\nu,\sigma r}}}_{\mathrm{\mathbb{P}^{\eta,\sigma v}}}} \eta(x)
\end{equation}
where $r$ is the drift for a reference process:
\begin{equation}
\label{eqn:reference_sde}
dX_t=\sigma r_t(X_t)\,dt+\sigma \overrightarrow{dW_t}, \; X_0\sim \nu\Rightarrow (X_t)_{t\in[0,T]} \sim  \overrightarrow{\mathbb{P}}^{\nu,\sigma r}
\end{equation}
and $\eta=\overrightarrow{\mathbb{P}}^{\nu,\sigma r}_T$. 
Via Girsanov's theorem, the method amounts to minimizing loss of the following type: $\mathcal{L}_{KL}(u)=$
\small
\begin{align}
&\mathbb{E}_{\overrightarrow{\mathbb{P}}^{\nu,\sigma u}}\Big[\log\Big(\frac{d\overrightarrow{\mathbb{P}}^{\nu,\sigma u}}{d\overleftarrow{\mathbb{P}}^{\mu,\sigma v}}\Big)\Big] = \mathbb{E}_{\overrightarrow{\mathbb{P}}^{\nu,\sigma u}}\Big[\log\Big(\frac{d\overrightarrow{\mathbb{P}}^{\nu,\sigma u}}{d\overrightarrow{\mathbb{P}}^{\nu,\sigma r}}\frac{d\overleftarrow{\mathbb{P}}^{\eta,\sigma v}}{d\overleftarrow{\mathbb{P}}^{\mu,\sigma v}}\Big)\Big] \nonumber\\
&= \mathbb{E}_{\overrightarrow{\mathbb{P}}^{\nu,\sigma u}}\Big[\log\Big(\frac{d\overrightarrow{\mathbb{P}}^{\nu,\sigma u}}{d\overrightarrow{\mathbb{P}}^{\nu,\sigma r}}\frac{d\eta}{d\mu}\Big)+\log Z\Big]\nonumber\\
&= \mathbb{E}_{X\sim\overrightarrow{\mathbb{P}}^{\nu,\sigma u}}\Big[\int_0^T \frac{1}{2}\|u_s(X_s)-r_s(X_s)\|^2ds+ 
\log\left(\frac{d\eta}{d\mu}\right)(X_T)\Big] \nonumber\\
&\quad \quad +\log Z\, . \label{eqn:girsanov_ref}
\end{align}
\normalsize
This suggests initializing from $\nu$ to estimate the loss \eqref{eqn:girsanov_ref} with the current control $u^{\hat{\theta}}$ by simulating \eqref{eqn:forward}, followed by gradient descent to optimize the $\hat{\theta}$-parameterized control and iterating between the two steps can be a viable strategy, that could identify $u^{\theta^*}$ eventually and used to run \eqref{eqn:forward} to draw samples from $\mu$. Note $\eta,\nu$ are simple distributions we have the freedom to pick, along with $r(\cdot)$ in \eqref{eqn:ref_pic}. The judicious choice of the common $v$ and $\nu$ in \eqref{eqn:sample_pic} and \eqref{eqn:ref_pic} allow \eqref{eqn:girsanov_ref} to take on a control-theoretic interpretation \citep{berner2022optimal}, where the cost \eqref{eqn:girsanov_ref} is composed of a running control cost and a terminal cost.  

For a concrete example consider $\nu=\mathcal{N}(0,\sigma^2 I), r_t(x) = -x/2\sigma$ in \eqref{eqn:reference_sde}, then process \eqref{eqn:ref_pic} is simply OU in equilibrium, i.e., $\nu=\eta$ \citep{vargas2022denoising}. The purpose of introducing the reference process \eqref{eqn:ref_pic} is to fix the backward drift $v$ in \eqref{eqn:sample_pic} so that $\min_u \, \mathcal{L}_{KL}(u)$ from \eqref{eqn:girsanov_ref} is unique. However, this example illustrates that in general unless $T\rightarrow \infty$, $\mathcal{L}_{KL}(u)$ in \eqref{eqn:girsanov_ref} can't be minimized to $0$ since $\overleftarrow{\mathbb{P}}^{\mu,\sigma v}_0\neq \nu$ under the OU process with $v_t(x)=x/2\sigma$ unless $T\rightarrow \infty$.  

\begin{remark}
The loss \eqref{eqn:girsanov_ref} enforces uniqueness (and correct marginals $\nu, \mu$) if minimized to $0$, but doesn't impose optimality of the interpolating trajectory in any way. The approach in \citep{chen2021likelihood} that relies on training two drifts amounts to solving $\min_{u,v}$ $D_{KL}(\overrightarrow{\mathbb{P}}^{\nu,f+\sigma u}||\overleftarrow{\mathbb{P}}^{\mu,f+\sigma v})$ jointly, and as shown in \citep{richter2023improved}, although ensure correct marginals, has non-unique minimizers. A unique solution is desirable since it ensures robustness - one gets the same marginal density trajectory regardless of the initialization/training procedure.     
\end{remark}

\paragraph{Reverse KL} In diffusion generative modeling, score-matching-based loss \citep{hyvarinen2005estimation} can be seen as minimizing the reverse KL over $s:=(u-v)/\sigma$ by running the backward process \eqref{eqn:backward} using samples from $\mu$ to estimate the loss. More concretely, using the Radon-Nikodym derivative in Lemma~\ref{lem:RN}, it gives $\mathcal{L}_{KL}(s)=$ 
\small
\begin{align}
&\mathbb{E}_{\overleftarrow{\mathbb{P}}^{\mu,\sigma v}}\Big[\log\Big(\frac{d\overleftarrow{\mathbb{P}}^{\mu,\sigma v}}{d\overrightarrow{\mathbb{P}}^{\nu,\sigma u}}\Big)\Big] =\mathbb{E}_{\overleftarrow{\mathbb{P}}^{\mu,\sigma v}}\Big[\log\Big(\frac{d\overleftarrow{\mathbb{P}}^{\mu,\sigma v}}{d\overrightarrow{\mathbb{P}}^{\nu,\sigma v+\sigma^2 s}}\Big)\Big] = \nonumber\\
&\mathbb{E}_{\overleftarrow{\mathbb{P}}^{\mu,\sigma v}}\Big[\int_0^T\frac{\sigma^2}{2}\|s_t(X_t)\|^2\, dt+\sigma^2\int_0^T \nabla\cdot s_t(X_t)\, dt\Big]+C = \nonumber\\
&\mathbb{E}_{\overleftarrow{\mathbb{P}}^{\mu,\sigma v}}\Big[\int_0^T \frac{\sigma^2}{2}\|s_t(X_t)-\nabla_x \log p_{t}^{\mu,\sigma v}(X_t|X_T=x_T)\|^2\, dt\Big] \nonumber\\
&\quad\quad +C' \label{eqn:score_matching_loss}\, .
\end{align}
\normalsize
where $C,C'$ is independent of $s$. Here one fixes the backward drift $v$ in $\overleftarrow{\mathbb{P}}^{\mu,\sigma v}$ so that $\nu$ is easy to sample from (e.g., an OU process). In the last transition above, we used the integration by parts identity $\mathbb{E}_{\rho_t}[\int_0^T s_t^\top \nabla \log \rho_t \, dt]=-\mathbb{E}_{\rho_t}[\int_0^T \nabla \cdot s_t \, dt]\, .$
In practice crucially there will be an irreducible loss since the terminals $\nu$ and $\overleftarrow{\mathbb{P}}_0^{\mu,\sigma v}$ don't match exactly for any finite $T$, but the dynamics $\overrightarrow{\mathbb{P}}^{\nu,\sigma v+\sigma^2 s}$ still makes sense as for two processes $(p_t)_t,(q_t)_t$ with different initializations $\nu, \overleftarrow{\mathbb{P}}_0^{\mu,\sigma v}$ that share the same drift $\sigma v+\sigma^2 s^*$, $\partial_t\, D_{KL}(p_t\Vert q_t)=-\mathbb{E}_{p_t}[\|\nabla \log\frac{p_t}{q_t}\|^2]\leq 0$ contracts (c.f. Remark~\ref{rmk:contract}), although we have not been quantitative about the rate. It is worth noting that such approach doesn't require re-generating trajectories iteratively as \eqref{eqn:girsanov_ref} does.

We therefore see that both the approach of \eqref{eqn:girsanov_ref} and \eqref{eqn:score_matching_loss} rely on fixing some aspect of the process in \eqref{eqn:sample_pic} to restore uniqueness of the loss $\mathcal{L}_{KL}$, but this choice is mostly out of convenience. In both cases, it results in an \emph{one-parameter loss}, and the minimizer is not affected by the unknown constant $Z$. Successfully optimizing $D_{KL}(\overrightarrow{\mathbb{P}}^{\nu,\sigma u}\Vert\overleftarrow{\mathbb{P}}^{\mu,\sigma v})$ or $D_{KL}(\overleftarrow{\mathbb{P}}^{\mu,\sigma v}\Vert \overrightarrow{\mathbb{P}}^{\nu,\sigma u})$ to zero error will imply that $u$ pushes $\nu$ to $\mu$ and $v$ vice versa, mimicking a noising/denoising reversible procedure. In what follows in Section~\ref{sec:main_sec}, we deviate from these perspectives by adopting a control formulation that does not rely on mixing of stochastic processes for transporting between two distributions, albeit still working with a pathwise formulation.

\section{METHODOLOGY}
\label{sec:main_sec}
As we saw from Section~\ref{sec:framework}, (1) relying on mixing property of diffusion process can make the trajectory rather long; (2) there are many degrees of freedom in transporting $\nu$ to $\mu$, neither of which is desirable for training purpose. This gives us the motivation to turn to losses based on path measures that can enforce a canonical choice (c.f. Appendix \ref{appsec:motivation} for additional practical motivation behind SB vs. other pathwise samplers). We will adopt an ``optimal control" perspective and leverage special properties of the SB problem to come up with valid control objectives for this effort. Consider over path space $\mathcal{C}([0,T];\mathbb{R}^d)$, given a reference measure $Q$, the \emph{constrained} optimization problem
\begin{equation}
\label{eqn:bridge_problem}
P^*=\arg\min_{P_0=\nu,P_T=\mu}\, D_{KL}(P\Vert Q)
\end{equation}
where we assume $Q$ admits the SDE representation (this can be thought of as a prior)
\begin{equation}
\label{eqn:Q_sde}
dX_t = f_t(X_t)\, dt+\sigma dW_t\,,\quad X_0\sim \nu\, ,
\end{equation}
which is a slight generalization of the classical case where typically $f=0$. Furthermore the path measure $P$ is assumed to correspond to the following SDE and we are interested in finding the optimal control $\phi$ in 
\begin{equation}
\label{eqn:sb_control_sde}
dX_t = [f_t(X_t)+\sigma^2 \nabla \phi_t(X_t)]\, dt+\sigma dW_t\,,\quad X_0\sim\nu
\end{equation}
that solves \eqref{eqn:bridge_problem}. Various perspectives on the Schr\"odinger Bridge problem \eqref{eqn:bridge_problem}-\eqref{eqn:Q_sde}, which we heavily leverage in the next part for designing our losses, are included in Appendix \ref{sec:SB}, along with related methods for solving the problem in different setups than what we consider in Appendix \ref{sec:attempt}.

\subsection{Training Loss Proposal}
\label{sec:proposal_loss}
Recall our focus is on solving the regularized optimal transport problem \eqref{eqn:bridge_problem} between $\nu$ and $\mu$, by learning the controls on the basis of samples from $\nu$. Out of many joint couplings with correct marginals (i.e., transport maps), the choice of a reference process will select a particular trajectory $\rho_t$ between $\nu$ and $\mu$. The execution of this plan crucially hinges on two ingredients: (1) general backward / forward likelihood ratio formula given in Lemma \ref{lem:RN}; (2) properties of the optimal drifts for SB from Section~\ref{sec:SB}, some of which can be exploited for training the controls. Challenge, as emphasized before, is we need to be able to estimate the resulting loss and ensuring that successful optimization guarantees convergence to the unique solution dictated by the SB. 
%
%
%
Proposition \ref{prop:loss} below serves as our main result, where we show that adding appropriate regularization can accomplish these. All proofs are deferred to Appendix~\ref{app:proof}.\\

\begin{strip}

\begin{proposition}[Control Training Objective]
\label{prop:loss}
For the problem of \eqref{eqn:bridge_problem}, the following losses are valid: \\

(a) $\arg\min_{\nabla \phi,\nabla \psi}\, D_{KL}(\overrightarrow{\mathbb{P}}^{\nu,f+\sigma^2 \nabla \phi}\Vert\overleftarrow{\mathbb{P}}^{\mu,f-\sigma^2\nabla \psi})+\lambda \cdot \mathbb{E}_{X\sim \overrightarrow{\mathbb{P}}^{\nu,f+\sigma^2 \nabla \phi}}\big[\int_0^T \frac{\sigma^2}{2} \|\nabla \phi_t(X_t)\|^2 dt\big]$

(b) $\arg\min_{\nabla \phi,\nabla \psi}\, \text{Var}_{\overrightarrow{\mathbb{P}}^{\nu,f+\sigma^2 \nabla \phi}}\left[\log\left(\frac{\overrightarrow{\mathbb{P}}^{\nu,f+\sigma^2 \nabla \phi}}{\overleftarrow{\mathbb{P}}^{\mu,f-\sigma^2\nabla \psi}}\right)\right] +$
\begin{align*}
&\text{Var}_{X\sim\overrightarrow{\mathbb{P}}^{\nu,f}}\Big(\psi_T(X_T)-\psi_0(X_0)+\int_0^T (-\frac{\sigma^2}{2}\|\nabla \psi_t\|^2+\nabla \cdot f_t-\sigma^2\Delta  \psi_t) (X_t)\, dt-\sigma\int_0^T \nabla \psi_t(X_t)^\top dW_t \Big)\\
\text{or}\;\;  &\text{Var}_{X\sim\overrightarrow{\mathbb{P}}^{\nu,f}}\Big(\phi_T(X_T)-\phi_0(X_0)+\frac{\sigma^2}{2} 
\int_0^T\|\nabla \phi_t\|^2(X_t)\, dt-\sigma\int_0^T \nabla \phi_t(X_t)^\top dW_t \Big)
\end{align*}

(c) $\arg\min_{\phi,\psi}\, \text{Var}_{X\sim \overrightarrow{\mathbb{P}}^{\nu,f+\sigma^2 \nabla \phi}}\left((\phi_T+\psi_T-\log \mu)(X_T)\right)+\text{Var}_{X\sim \overrightarrow{\mathbb{P}}^{\nu,f+\sigma^2 \nabla \phi}}\left((\phi_0+\psi_0-\log\nu)(X_0)\right)+$ 
\begin{align*}
&\text{Var}_{X\sim \overrightarrow{\mathbb{P}}^{\nu,f+\sigma^2 \nabla \phi}}\Big(\phi_T(X_T)-\phi_0(X_0)-\frac{\sigma^2}{2} 
\int_0^T \|\nabla \phi_t\|^2(X_t)\, dt-\sigma\int_0^T \nabla \phi_t(X_t)^\top dW_t \Big)+ \\
&\text{Var}_{X\sim \overrightarrow{\mathbb{P}}^{\nu,f+\sigma^2 \nabla \phi}}\Big(\psi_T(X_T)-\psi_0(X_0)-\int_0^T 
\Big(\frac{\sigma^2}{2}\|\nabla \psi_t\|^2+\nabla \cdot (\sigma^2\nabla \psi_t-f_t)+\sigma^2\nabla \psi_t^\top\nabla\phi_t\Big) (X_t)\, dt-\sigma\int_0^T \nabla \psi_t(X_t)^\top dW_t \Big)
\end{align*}

(d) $\arg\min_{\nabla \phi_t,\nabla \log \rho_t}\, D_{KL}(\overrightarrow{\mathbb{P}}^{\nu,f+\sigma^2\nabla \phi_t}\Vert\overleftarrow{\mathbb{P}}^{\mu,f+\sigma^2\nabla \phi_t-\sigma^2\nabla \log\rho_t})+\lambda\cdot \mathbb{E}_{\overrightarrow{\mathbb{P}}^{\nu,f+\sigma^2\nabla \phi}}\big[\int_0^T \frac{\sigma^2}{2}\|\nabla \phi_t(X_t)\|^2 \, dt\big]$\\

In particular, if the loss is minimized to $0$ for (b) and (c), the resulting $\nabla \phi^*$ solves the SB problem \eqref{eqn:bridge_problem} from $\nu$ to $\mu$, that can in turn be used for sampling from $\mu$ by running \eqref{eqn:sb_control_sde}. In all cases, $X_0=x_0\sim\nu$ is assumed given as initial condition. Moreover, $\nabla \psi^*$ is the corresponding backward drift that drives $\mu$ to $\nu$.\\
\end{proposition}

\end{strip}


Both (a) and (d) are guided by the ``minimum-action" principle w.r.t a reference (i.e., minimum control energy spent). (b) bases itself on a reformulation of the HJB PDE involving the optimal control, and (c) is grounded in the FBSDE system for SB optimality \citep{chen2021likelihood} (c.f. Remark~\ref{rmk:equivalence} for connections between particular SDEs, PDEs and path measures). In all cases, the objective is a \emph{two-parameter loss} that also allows the recovery of the score $\nabla \log \rho_t$. Note that variance is taken w.r.t the uncontrolled process in (b) and w.r.t the controlled process in (c). Losses (a) and (d), being conversion from the constrained problem (and each other), need $\lambda$ to be picked relatively small so that the first part of the objective $=0$ to identify the unique SB solution, but we establish a bound on the optimal objective value in Remark~\ref{rmk:bound}. The proof of these results heavily use the factorization property \eqref{eqn:factorize_condition} satisfied by the optimal coupling, and exploit different ways to encode the optimality condition. 

Loss (c) is different from log-variance divergence over path space $\text{Var}_{\overrightarrow{\mathbb{P}}^{\nu,f+\sigma^2\nabla \phi}}[\log(\frac{d\overrightarrow{\mathbb{P}}^{\nu,f+\sigma^2\nabla \phi}}{d\overleftarrow{\mathbb{P}}^{\mu,f-\sigma^2\nabla \psi}})]\, ,$
which will \emph{not} guarantee finding the optimal path, whereas the objective in (c) incorporates the dynamics of \emph{two} controlled (and coupled) dynamics w.r.t $\overrightarrow{\mathbb{P}}^{\nu,f}$ that we know how to characterize optimality for using results in Section~\ref{sec:SB}. More specifically, we \emph{separately} impose optimality condition on 
\small\begin{equation}
\label{eqn:separate_path}
\log \left(\frac{d\overrightarrow{\mathbb{P}}^{\nu,f+\sigma^2\nabla \phi_t}}{d\overrightarrow{\mathbb{P}}^{\nu,f}}\right)(X)\; \text{and} \;\log\left(\frac{d\overrightarrow{\mathbb{P}}^{\nu,f}}{d\overleftarrow{\mathbb{P}}^{\mu,f-\sigma^2\nabla \psi_t}}\right)(X)
\end{equation}\normalsize
for $X\sim \overrightarrow{\mathbb{P}}^{\nu,f+\sigma^2\nabla \phi_t}$ to take factorized forms, as opposed to looking at divergence metrics on their sum $\log\left(\frac{d\overrightarrow{\mathbb{P}}^{\nu,f+\sigma^2\nabla \phi}}{d\overleftarrow{\mathbb{P}}^{\mu,f-\sigma^2\nabla \psi}}\right)$
only, effectively erasing the ``SB optimality enforcement" part. 
 Another way to view the variance regularizers in loss (b) and (c) is through the SDE representation of the controls from Lemma~\ref{lem:sde_for_sb} and observe that the variance condition precisely encodes the optimally-controlled dynamical information. It is important to note that we are evaluating the change of $\phi,\psi$ on a particular stochastic trajectory, rather than tracking the cumbersome evolution of density over the full $\mathbb{R}^d$ space, as what the PDEs \eqref{eqn:sb_system} may suggest.

\begin{remark}[Stochastic gradient w.r.t controls at optimality]
\label{rem:variance}
In \citep{richter2023improved}, the authors show that log-variance divergence has the advantage of having variance of gradient = 0 at the optimal $\phi^*, \psi^*$, which is not true for $D_{KL}$ in general, and has consequence for gradient-based updates such as those in Algorithm~\ref{alg:algorithm}. Similar argument applies to the variance regularizer we consider (e.g., loss (b)). For this, we look at the G\^ateaux derivative of the variance function $V$ in an arbitrary direction $\tau$ since from the chain rule $\frac{\delta}{\delta \phi}V(\phi,\psi;\tau):=\frac{d}{d\epsilon}\vert_{\epsilon=0}V(\phi+\epsilon\tau,\psi)\Rightarrow
\partial_{\theta_i} V(\phi_\theta,\psi_\gamma)=\frac{\delta}{\delta \phi}\vert_{\phi=\phi_\theta} V(\phi,\psi_\gamma;\partial_{\theta_i} \phi_\theta).$ 
Now if $\hat{V}$ is the Monte-Carlo estimate of the variance of a random quantity $g(\cdot)$, it is always the case that (we use $\frac{\delta}{\delta \phi}(\cdot)_\tau$ to denote derivative in the $\tau$ direction)
\begin{align*}
&\frac{\delta}{\delta \phi} \hat{V}(\phi,\psi;\tau)=\frac{\delta}{\delta \phi}(\hat{\mathbb{E}}[g(\phi,\psi)^2]-\hat{\mathbb{E}}[g(\phi,\psi)]^2)_\tau\\
&=2\hat{\mathbb{E}}[g(\phi,\psi)\frac{\delta}{\delta \phi} g(\phi,\psi)_\tau]-2\hat{\mathbb{E}}[g(\phi,\psi)]\frac{\delta}{\delta \phi}\hat{\mathbb{E}}[g(\phi,\psi)]_\tau\,.
\end{align*}
Hence if $g(\phi^*,\psi^*)=\text{const}$ a.s. for every i.i.d sample, such as the regularizer in loss (b), the derivative w.r.t the control $\phi$ in direction $\partial_{\theta_i} \phi_\theta$ is $0$ at optimality, implying $\text{Var}(\partial_{\theta_i} \hat{V}(\phi_{\theta^*},\psi_{\gamma^*}))=0$ by the chain rule. 
\end{remark}


It's natural to ask if one can replace the variance regularizer $\text{Var}(\cdot)$ with a moment regularizer $\mathbb{E}[|\cdot|^2]$ in e.g., loss (c). However while the variance is oblivious to constant shift, the moment loss will require knowledge of the normalizing constant $Z$ of the target $\mu$ to make sense. In the case of loss (b), using the martingale property \eqref{eqn:martingale} and Lemma~\ref{lem:sde_for_sb}, an alternative proposal based on moment regularizer for e.g., $\phi$ can be 
\small\[\mathbb{E}_{X\sim\overrightarrow{\mathbb{P}}^{\nu,f}}\Big[\underbrace{\Big|\phi_T(X_T)-\phi_0(X_0)+\frac{\sigma^2}{2} 
\int_0^T\|\nabla \phi_t\|^2(X_t)\, dt \Big|^2}_{\bar{g}(\phi)^2} \Big]\, .\]\normalsize
However, the G\^ateaux derivative of the loss in this case takes the form of 
\[\frac{\delta}{\delta \phi}(\hat{\mathbb{E}}[\bar{g}(\phi)^2])_\tau=2\hat{\mathbb{E}}[\bar{g}(\phi)\frac{\delta}{\delta \phi} \bar{g}(\phi)_\tau]\, ,\]
which means that while $\mathbb{E}[\bar{g}(\cdot)]=0$ vanish for the optimal control $\phi^*$ in expectation, for individual trajectory $\bar{g}_i(\cdot)\neq 0$ generally at $\phi^*$, hence yielding non-vanishing variance for stochastic gradient and posing challenges for optimization. For our variance regularizer in loss (b) where $g_i(\phi):=\phi_T(X_T^i)-\phi_0(X_0^i)+\frac{\sigma^2}{2} 
\int_0^T\|\nabla \phi_t\|^2(X_t^i)\, dt-\sigma\int_0^T \nabla \phi_t(X_t^i)^\top dW_t$ is $0$ almost surely for every trajectory at $\phi^*$, it implies vanishing gradient and will therefore identify the optimal $\phi^*$ even with mini-batch updates. In terms of the actual minimum value of the empirical loss, our loss (b) obeys $\hat{V}[g(\phi^*)]=0$ whereas using It\^o's isometry $\hat{\mathbb{E}}[\bar{g}(\phi^*)^2]=\hat{\mathbb{E}}[\int\sigma^2\|\nabla \phi^*_t\|^2\, dt] \neq 0$ is much less predictable for monitoring the performance of the final control. This is somewhat similar to \citep{zhou2021actor} where the authors observe that including random $0$-mean martingale terms is important for variance reduction in a different context.


\begin{remark}[Comparisons]
In \citep{vargas2023transport}, the authors propose $\int_0^T \mathbb{E}\big| \partial_t \phi +  f^\top \nabla \phi+\frac{\sigma^2}{2}\Delta \phi+\frac{\sigma^2}{2}\|\nabla\phi\|^2 \big|(X_t) dt$ as the HJB regularizer (c.f. \eqref{eqn:hjb}) on top of $D_{KL}(\overrightarrow{\mathbb{P}}^{\nu,f+\sigma^2 \nabla \phi}\Vert\overleftarrow{\mathbb{P}}^{\mu,f-\sigma^2\nabla \psi})$ as the training loss, inspired by PINN \citep{raissi2019physics}. As is clear from the proof in Proposition~\ref{prop:loss}, it is equally valid for identifying the optimal drift as our loss (b) (see also related result in \citep{nusken2021interpolating}). Combining \eqref{eqn:hjb}-\eqref{eqn:FPK} we can also get a HJB for the backward drift $\nabla \psi$, which read as $\partial_t \psi+f^\top \nabla \psi +\nabla \cdot f-\frac{\sigma^2}{2}\Delta\psi-\frac{\sigma^2}{2}\|\nabla \psi\|^2=0$ for imposing the PINN loss. By trading a PDE constraint for a SDE one (based on likelihood ratio of path measures), we can avoid evaluating the divergence term, in additional to the benefit of ``sticking to the landing" (c.f. Remark~\ref{rem:variance}). We refer to Section~\ref{sec:experiment} for numerical illustration. 
\end{remark}

In Appendix \ref{sec:second_order}, we consider the extension of such training methodology to the case of second-order dynamics in the augmented $(X,V)$ space \`a la under-damped Langevin.

\subsection{Discretization \& Implementation}
\label{sec:implement}
We discuss implementation of our losses and estimators here. In practice, with imperfect controls from the training procedure, one can perform importance sampling to correct for the bias / improve on the estimate -- something only available for path-wise samplers working in extended state space. 
\begin{proposition}[Importance Sampling]
\label{prop:importance_sample}
The following can be used to give unbiased estimate of the normalizing constant for $p_{\text{target}}$:
\begin{enumerate}[label={(\arabic*)}]
\item For the optimal $\phi^*,\psi^*$, with $X_t\sim \overrightarrow{\mathbb{P}}^{\nu,f+\sigma^2 \nabla \phi^*}$,
$Z = \frac{\mu(X_T)}{\nu(X_0)}\exp(\frac{\sigma^2}{2}\int_0^T\nabla\cdot (\nabla \phi_t^*-\nabla \psi_t^*)(X_t)\, dt+\int_0^T \nabla \cdot f_t(X_t) dt)$\, .

\item For any suboptimal $\phi,\psi$, with $X_t\sim \overrightarrow{\mathbb{P}}^{\nu,f+\sigma^2 \nabla \phi}$,
$Z=\mathbb{E}[\exp(-\frac{\sigma^2}{2}\int_0^T \|\nabla \phi_t+\nabla \psi_t\|^2+\nabla \cdot(f_t-\sigma^2\nabla \psi_t)dt-\sigma\int_0^T \nabla \phi_t+\nabla \psi_t dW_t-\log\frac{\nu(X_0)}{\mu(X_T)})]$\, .
\end{enumerate}
With a sub-optimal control $\nabla \phi$ and $X_t\sim \overrightarrow{\mathbb{P}}^{\nu,f+\sigma^2 \nabla \phi}$, re-weighting can be used to get an unbiased estimator of a statistics $g\colon \mathbb{R}^d\rightarrow \mathbb{R}$ as 
\[\frac{\mathbb{E}_{ \phi}[g(X_T) w^{ \phi}(X_T)]}{\mathbb{E}_{ \phi}[w^{ \phi}(X_T)]} = \mathbb{E}_{ \phi^*}[g(X_T)]=\mathbb{E}_{p_\text{target}}[g]\]
with weight
\small\begin{equation*}
w^\phi(X)= \exp\Big(\int_0^T \sigma^2\Delta\phi_t-\frac{\sigma^2}{2}\Delta\log\rho_t+\nabla \cdot f_t\, dt\Big)\frac{d\mu (X_T^\phi)}{d\nu (X_0^\phi)}.
\end{equation*}\normalsize
\end{proposition}

As in \cite[Proposition F.1]{vargas2023transport}, it is possible to trade divergence term for a backward integral when dealing with path integrals using \eqref{eqn:conversion}. 
%
%
%
%
Thanks to the fact that our various estimators and regularizers are built upon path measures, the following lemma generalizes this idea and provides a recipe for estimating the regularizer from Proposition~\ref{prop:loss} and the normalizing constant from Proposition~\ref{prop:importance_sample} with discrete-time updates that are cheap to evaluate. We work out loss (c) from Proposition~\ref{prop:loss} below -- most other parts are straightforward to adapt so we simply state them in Section~\ref{sec:spec}. 
\begin{strip}
\begin{lemma}[Discretized Loss and Estimator]
\label{lem:discretization}
For $X\sim \overrightarrow{\mathbb{P}}^{\nu,f+\sigma^2 \nabla \phi}$, the last part of loss $(c)$ on $\psi$
can be estimated 
\[\text{Var}_N\Big(\psi_K(X_{K+1}^i)-\psi_0(X_0^i)+ \frac{1}{2\sigma^2 h}\sum_{k=0}^{K-1}  \|X_{k+1}^i-X_k^i-f_{k}(X_k^i)h\|^2-\|X_{k}^i-X_{k+1}^i+(f_{k+1}-\sigma^2\nabla \psi_{k+1})(X_{k+1}^i)h\|^2\Big),\] where $\text{Var}_N$ denotes empirical estimate of the variance using $N$ samples.

The importance-weighted $Z$-estimator from Proposition~\ref{prop:importance_sample} can be approximated as 
\begin{equation}
\label{eqn:discrete_z_estimate}
\hat{Z}=\frac{1}{N}\sum_{i=1}^N \exp\Big(\log\frac{\mu(X_K^i)}{\nu(X_0^i)}+\frac{1}{2}\sum_{k=0}^{K-1}\|z_k^i\|^2- \frac{1}{2\sigma^2 h}\|X_k^i-X_{k+1}^i+(f_{k+1}-\sigma^2\nabla \psi_{k+1})(X_{k+1}^i)h\|^2\Big)\, .
\end{equation}

In both cases for $i\in[N]$ independently, $z_k^i\sim\mathcal{N}(0,I)$, $X_{k+1}^i=X_k^i+(f_k(X_k^i)+\sigma^2 \nabla \phi_k(X_k^i))h+\sigma\sqrt{h} \cdot z_k^i\, .$
\end{lemma}

\end{strip}

Putting everything together gives the final algorithm. 
\begin{algorithm}[H]
\caption{Control Objective Training for Sampling from Un-normalized Density}\label{alg:algorithm}
\begin{algorithmic}
\Require Initial draw $(X_0^{i,(0)})_{i=1}^{N} \in\mathbb{R}^d \sim \nu$ independent, initial controls $\phi^{(0)},\psi^{(0)}$ 
\Require Un-normalized density $\mu$, base drift $f$, num of time steps $K$, num of iterations $T$
\For{$t=0,\cdots,T-1$}
\State Run \eqref{eqn:em-discretize} with current control $(\nabla \phi^{(t)}_k)_{k=0,\cdots, K}$ to obtain $(X_k^{n,(t)})_{k=0,\cdots, K}$ for $n=1,\cdots N$
\State Estimate the loss \eqref{eqn:log_divergence} $+$ discretized regularizer (c.f. Lemma~\ref{lem:discretization} \& Section~\ref{sec:spec}) and the gradient w.r.t the two parameterized controls using the samples $(X_k^{n,(t)})_{k=0,\cdots K, 
 n=1,\cdots, N}$
\State Gradient update on the parameters to obtain $\nabla \phi^{(t+1)}$ and $\nabla \psi^{(t+1)}$
\EndFor
\State\Return $X_K^{1,(T)},\cdots, X_K^{N,(T)}$ as $N$ samples from $\mu$ with their importance weights $w^{\phi^{(T)}}(X_K^{n,(T)})$ \eqref{eqn:discretized_w}, and the weighted $Z$ estimator \eqref{eqn:discrete_z_estimate} for $\mu$
\end{algorithmic}
\end{algorithm}

We also draw a connection between the optimal $Z$ estimator and the optimal controls below. 
\begin{remark}[Optimal $\log Z$ estimator]
As observed in \citep{vargas2023transport}, discretizations with backward integral have the additional benefit of giving ELBO lower bound for the normalizing constant $Z$ of $p_{\text{target}}$. Since our $\hat{Z}$-estimator from Lemma~\ref{lem:discretization} can be understood as a ratio of two discrete chains 
\begin{equation*}
\hat{Z}=\frac{\mu(X_K)q^v(X_{0:K-1} | X_K)}{\nu(X_0)p^u(X_{1:K}|X_0)}\, ,
\end{equation*}
it implies that $\mathbb{E}_{\nu(X_0)p^u(X_{1:K}|X_0)}[\log \hat{Z}]\leq \log (\mathbb{E}_{\nu(X_0)p^u(X_{1:K}|X_0)}[\hat{Z}]) =\log\left[\int \mu(X_K) \, dX_K\right] =\log(Z)$\,.
The estimator is reminiscent of the philosophy adopted in annealed importance sampling (AIS) with extended target, but in our case, the backward kernel $q^v(\cdot)$ is the \emph{time-reversal} of the forward one, which can be shown to be the optimal transition kernel minimizing the variance of the resulting evidence estimate \citep{doucet2022score}. In addition, the forward kernel $p^u(\cdot)$ in our case follows an optimal trajectory that ensures the chain reaches the \emph{exact} target $\mu$ rapidly. 
\end{remark}

\section{NUMERICS \& COMPARISONS}
\label{sec:experiment}
In this section, we instantiate our main contributions (Proposition \ref{prop:loss} and Lemma \ref{lem:discretization}) and offer numerical evidence on their advantages compared to existing proposals for solving an optimal trajectory problem in a typical MCMC setup. We highlight that our algorithm does not use gradient information from the target $\nabla \log \mu$ as e.g., Langevin would.

\subsection{Algorithm Specification} 
\label{sec:spec}
Below for simplicity we pick the reference process to be a Brownian motion with $f=0$ and $\lambda$ is a parameter that we tune for best performance, but in theory any $\lambda>0$ would work. This aspect deviates from other constrained formulation of the problem that may require special choice of $\lambda$. 

(1) PINN-regularization \citep{vargas2023transport}: for $i=1,\cdots,n$, and $Z\sim\mathcal{N}(0,I)$ independently draw in parallel
\begin{equation}
\label{eqn:em-discretize}
x_{k+1}^i = x_k^i+\sigma^2 h\nabla \phi(x_k^i,kh)+\sigma \sqrt{h}Z_k^i,\, x_0^i\sim \nu
\end{equation}
for $k=0,\cdots,K$ with stepsize $h=c/(K+1)$ for some $c\geq 1$. Using trajectories $\{x_k^i\}$, minimize over $\phi,\psi$,

\small
\begin{align}
&\frac{1}{K+1}\text{Var}_n \Big[\log \frac{\nu(x_0^i)}{\mu(x_{K+1}^i)}+\nonumber\\
&\sum_{k=0}^K \frac{1}{2\sigma^2 h}(\|x_k^i-x_{k+1}^i- 
\sigma^2h\nabla \psi(x_{k+1}^i,(k+1)h)\|^2 \nonumber \\
&- \|x_{k+1}^i-x_k^i-\sigma^2 h \nabla \phi(x_k^i,kh)\|^2)\Big]+  \label{eqn:log_divergence}\\ 
&\frac{\lambda h}{n}\sum_{i=1}^n \sum_{k=0}^{K} \Big|\partial_t \phi(x_k^i,kh)+
\frac{\sigma^2}{2}\Delta \phi(x_k^i,kh)+\frac{\sigma^2}{2}\|\nabla \phi(x_k^i,kh)\|^2\Big|  \label{pinn:loss}
\end{align}\normalsize 
The first term \eqref{eqn:log_divergence} is an estimate of the log-variance divergence between the two path measures. Above $\phi(x,t), \psi(x,t)$ are two neural networks that take $t\in \mathbb{R}$ and $x\in \mathbb{R}^d$ as inputs and map to $\mathbb{R}$. We repeat \eqref{eqn:em-discretize} and \eqref{pinn:loss} several times, and compute statistics using samples $\{x_{K+1}^i\}_{i=1}^n$ at the end. 


(2) Variance-regularization (loss (b) from Proposition \ref{prop:loss}): Simulate trajectories \eqref{eqn:em-discretize} as before, additionally simulate $\{y_k^i\}_{k=0}^K$ as follows and cache them: 
\begin{equation}
\label{eqn:em-discretize_uc}
y_{k+1}^i = y_k^i+\sigma \sqrt{h}Z_k^i,\, y_0^i\sim \nu\, .
\end{equation}
%
Minimize over $\phi,s$ the following discretized loss
\small
\begin{align}
&\eqref{eqn:log_divergence}+ \frac{\lambda}{K+1}\cdot  \text{Var}_n \Big[\phi(y_{K+1}^i,(K+1)h)-\phi(y_0^i,0)+\frac{1}{2\sigma^2h}\nonumber\\
&\sum_{k=0}^K 
\|y_{k+1}^i-y_k^i-\sigma^2\nabla \phi(y_k^i,kh)h\|^2-\|y_{k+1}^i-y_k^i\|^2\Big] \label{var:loss}
\end{align} \normalsize
Alternate between \eqref{eqn:em-discretize} and \eqref{var:loss} several times. Above $\text{Var}_n$ denotes the empirical variance across the $n$ trajectories $\{y_k^i\}_{i=1}^n$ of the quantity inside $[\cdot]$. The loss \eqref{var:loss}, as the proof of Proposition \ref{prop:loss} shows, comes from the fact that along the prior $X\sim\overrightarrow{\mathbb{P}}^{\nu,f}$, if $\phi$ is optimal,
\begin{align}
&\log \Big(\frac{d\overrightarrow{\mathbb{P}}^{\nu,f+\sigma^2\nabla \phi_t}}{d\overrightarrow{\mathbb{P}}^{\nu,f}}\Big)(X) = \label{eqn:crucial_condition}\\
&\int_0^T -\frac{\sigma^2}{2}\|\nabla \phi_t\|^2 dt + \int_0^T \sigma\nabla \phi_t^\top \,dW_t 
 {\color{red}{ \overset{!}{=} \phi_T(X_T)-\phi_0(X_0)}}\nonumber
\end{align}
has to satify the factorization \eqref{eqn:factorize_condition-1}. And we discretized the Radon-Nikodym derivative \eqref{eqn:crucial_condition} similar to how it was done in the KL divergence $D_{KL}(\overrightarrow{\mathbb{P}}^{\nu,\sigma^2 \nabla \phi}\Vert\overleftarrow{\mathbb{P}}^{\mu,-\sigma^2\nabla \psi})$ (c.f. Lemma \ref{lem:discretization} for a similar derivation).

(3) Instead of the regularizer \eqref{var:loss}, another discretization of loss (b) using condition \eqref{eqn:crucial_condition} can be a LSTD-like regularizer similar in spirit to \citep{liu2022deepRL}:
\begin{align}
&\eqref{eqn:log_divergence}+\lambda\cdot \frac{h}{n}\sum_{i=1}^n \sum_{k=0}^K \Big\vert\phi(y_{k+1}^i,(k+1)h)-\phi(y_{k}^i,kh)+\nonumber\\
&\frac{\sigma^2 h}{2} \|\nabla \phi(y_{k}^i,kh)\|^2-\sigma\sqrt{h}\nabla \phi(y_{k}^i,kh)^\top Z_k^i\Big\vert\, , \label{eqn:TD}
\end{align}
where the $Z_k^i$'s are re-used from \eqref{eqn:em-discretize_uc}.
The loss above can be justified with Lemma \ref{lem:sde_for_sb}. 

For the first three losses that we experiment \eqref{pinn:loss}, \eqref{var:loss}, \eqref{eqn:TD}, an analogous regularization on the backward drift involving $\nabla \psi$ is also possible (c.f. Remark~\ref{rmk:backward_drift}).

%

(4) Separately-controlled loss (loss (c) from Proposition \ref{prop:loss}): Simulate \eqref{eqn:em-discretize} as before, with the $n$ trajectories $\{x_k^i\}$, minimize over $\phi,\psi$ the following discretized loss (c.f. Lemma \ref{lem:discretization}): 
\small
\begin{align}
&\text{Var}_n[\psi(x_{K+1}^i,(K+1)h)+\phi(x_{K+1}^i,(K+1)h)-\log \mu(x_{K+1}^i)] + \nonumber\\
&\text{Var}_n[\psi(x_0^i,0)+\phi(x_0^i,0)-\log \nu(x_0^i)]+ \label{eqn:new_loss}\\ 
&\frac{\lambda}{K+1}  \text{Var}_n \Big[\psi(x_{K+1}^i,(K+1)h)-\psi(x_0^i,0)+\frac{1}{2\sigma^2h}\nonumber\\
&\sum_{k=0}^K  \|x_{k+1}^i-x_k^i\|^2 - \|x_{k}^i-x_{k+1}^i-\sigma^2 h\nabla \psi(x_{k+1}^i,(k+1)h)\|^2 \Big] \nonumber\\
&+\frac{\lambda}{K+1}\text{Var}_n\Big[ \phi(x_0^i,0)-\phi(x_{K+1}^i,(K+1)h)+\frac{1}{2\sigma^2 h}\nonumber\\
&\sum_{k=0}^K  
\|x_{k+1}^i-x_k^i\|^2-\|x_{k+1}^i-x_k^i-\sigma^2 h\nabla \phi(x_k^i,kh)\|^2\Big] .  \nonumber
\end{align}
\normalsize
It consists of $4$ variance terms, the last $2$ coming from discretization of path measures \eqref{eqn:separate_path}. We alternate between simulating \eqref{eqn:em-discretize} and updating $\phi,\psi$ from \eqref{eqn:new_loss}. One can also consider SDE-based discretization for the last 2 terms but will incur additional Laplacian and divergence terms as suggested by Lemma \ref{lem:sde_for_sb}, \eqref{eqn:fbsde_control_1}-\eqref{eqn:fbsde_control_2}. 



\paragraph{Comparison between discretized losses} We emphasize that the discretized variance regularizers \eqref{var:loss} and \eqref{eqn:backward_var} wouldn't be available without the path-wise stochastic process perspective. The FBSDE view (Lemma \ref{lem:sde_for_sb}) will naturally lend to TD-like regularizers \eqref{eqn:TD} and \eqref{eqn:backward_td} similar to \citep{liu2022deepRL}, but the crucial differences are: (1) we base the estimate on reference $\overrightarrow{\mathbb{P}}^{\nu,f}$ therefore there's no need to differentiate through the generated trajectory when optimize the loss over $\nabla \phi$ or $\nabla \psi$; (2) \citep{liu2022deepRL} consider dynamics with mean-field interaction and different loss. The two discretized objectives \eqref{var:loss}/\eqref{eqn:backward_var} and \eqref{eqn:new_loss} both enjoy variance reduction property as elaborated in Remark~\ref{rem:variance}, without the need to evaluate expensive Laplacian terms (as the PDE-based PINN approach \eqref{pinn:loss}/\eqref{eqn:backward_pinn} or TD approach \eqref{eqn:backward_td} would require). 

\subsection{Result}
In our experiments, we picked the following 4 targets as benchmark as they capture different properties of distributions that can affect the training of the controls. These properties (multimodality, spatially varying curvature) pose challenges for MCMC methods and are ideal examples to demonstrate the differences between the regularizers.
\vspace{-0.2cm}
\begin{figure}[h]
   \centering
        \includegraphics[width=0.8\hsize]{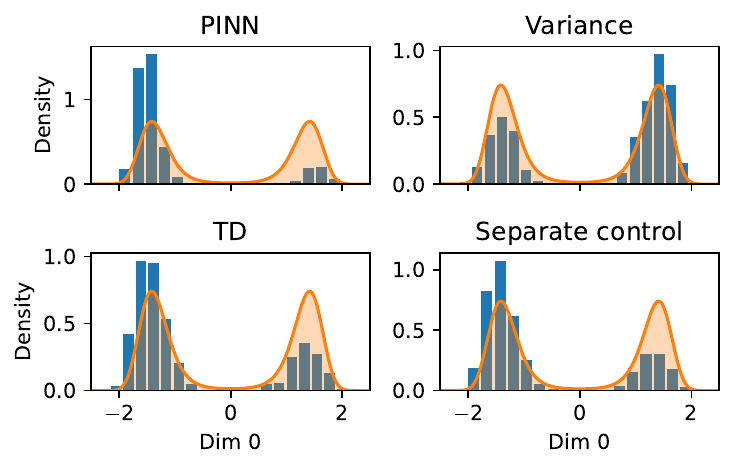}
    \caption{Weighted Marginal for Double Well}
\end{figure}
\vspace{-0.2cm}
\begin{table}[H]
\caption{$-\log Z$ Estimator (lower is better)}
\begin{center}
\begin{tabular}{@{}lcccc@{}}
\toprule
                & PINN                     & Variance & TD    & SC \\ \midrule
Normal & \num{6.317639350891113} & \num{2.8825278282165527}    & \num{2.013974666595459}  & \textbf{\num{1.3046903610229492}}            \\
Funnel          & \num{8.854369163513184} & \num{3.560049533843994}     & \num{3.5501890182495117}
 & \textbf{\num{3.390867233276367}}             \\
GMM             & \num{5.802513122558594} & \num{8.346628189086914}     & \textbf{\num{4.097251892089844}}  & \num{4.393266677856445}             \\
Double well & \num{12.274190902709961} & \num{5.391733169555664} & \num{4.436683654785156} & \textbf{\num{1.7614850997924805}} \\ \bottomrule
\end{tabular}
\end{center}
 \end{table}



\vspace{-0.2cm}
Additional supporting numerical experiments and details can be found in Appendix~\ref{app:numerics}. Across our experiments, we observe that our separately controlled loss \eqref{eqn:new_loss} is generally much better compared to the PINN loss \eqref{pinn:loss} that tends to exhibit mode-seeking behavior, while TD \eqref{eqn:TD} and variance regularizer \eqref{var:loss} can sometimes be comparable. We also observe that the separately controlled loss is less sensitive to tuning parameters, and the training loss curve is often smoother.

From a practical standpoint, there are additional benefits for basing the methodology on a SB formulation. Since we have the knowledge that the optimal control vector field (1) takes the gradient form of a scalar function (i.e., divergence free therefore conservative); (2) satisfy the HJB equation which admit e.g., rotation equivariance: if $\phi(x,t),f(x,t)$ is a solution, $\phi(Rx,t), R^\top f(Rx,t)$ is another solution for a rotation matrix $R$. So in the example $f_t(x_t)=-x_t$, it simply implies that $\phi_t(x)$ is a radial function. These can be baked into the architecture as inductive bias for training in practical implementation \citep{kondor2018clebsch,richter2022neural} so the NN can be sufficiently constrained to be more sample-efficient. This is an advantage compared to other diffusion-like samplers since this type of precise characterization of the optimal solution is something non-SB-bridge-based samplers do not admit.



\section{CONCLUSION}
We exploited the connections between diffusion generative modeling, stochastic control and optimal transport, with the goal of sampling from high-dimensional, complex distributions in mind. This is orthogonal to MCMC-based approaches, and is accomplished by a more ``learning-driven" methodology on the optimal control / drift that can be trained with a suitable control objective. While there are a lot of flexibility in the design (that also makes the case for SB more compelling than existing pathwise samplers mentioned in Section~\ref{sec:related}), we have illustrated that careful choice is needed for both numerical implementation and generalizability of the framework.  

More broadly, the mapping between target $p_{\text{target}}(x)$ and control $\phi(x,t)$ can be thought of as a form of operator learning mapping between functions in infinite dimensional space, for which universal approximation theorems have emerged recently for various neural network architectures \citep{de2022generic}. More specifically, one can view the task as learning the solution operator of the coupled PDEs \eqref{eqn:hjb}-\eqref{eqn:FPK}. This also points to the fact that our SB-based methodology can effectively leverage available data (i.e., supervised $\mu,\nabla \phi$ pairs) to generalize across tasks rather than being a task-specific problem only, such as the other diffusion-type samplers that train for path-measure consistency. Application-wise, exploiting such methods in the context of molecular dynamics simulation for sampling transition path is also exciting and development is already underway \citep{holdijk2024stochastic}. 

\bibliography{sample_paper}
\bibliographystyle{plainnat}

\appendix
\onecolumn
\clearpage
\thispagestyle{empty}
\aistatstitle{Control, Transport and Sampling Towards Better Loss Design: \\
Supplementary Materials}



\section{ADDITIONAL MOTIVATION \& PRACTICAL CONSIDERATION}
\label{appsec:motivation}

The main application we have in mind is for gradient-free sampling but the very same methodology can be applied for optimal transport and stochastic control applications. To give a concrete example, one might have handwriting of 2 digits collected from a group of individuals as the 2 marginals, the coupling $\pi^* (\nu,\mu)$ identified by the SB will put most of its mass on the digits written by the same people, therefore this kind of matching capability has applications beyond marginal sampling at the terminal time. The optimal control aspect of the SB trajectory finds application in e.g., transition path / rare-event sampling \citep{holdijk2024stochastic}, which is a well-known challenge in computational chemistry.
\begin{figure}[H]
\centering
\includegraphics[width=0.4\hsize]{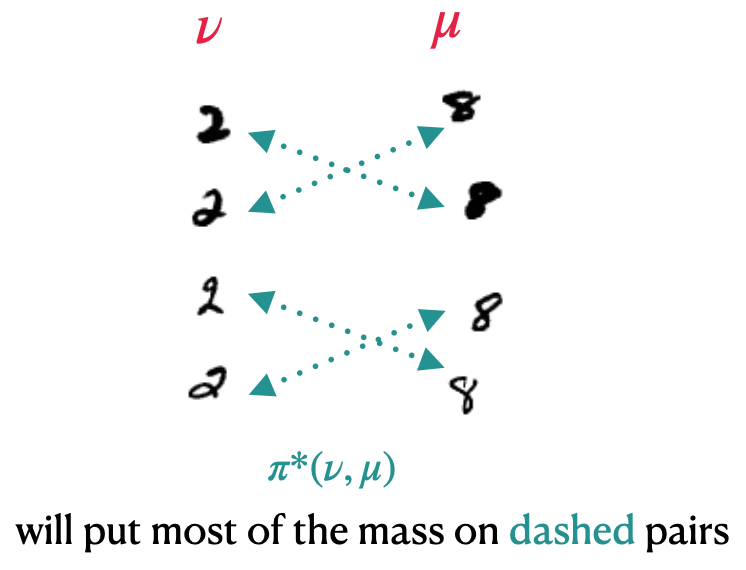}
\caption{Optimal Transport Between Fixed Marginals}
\end{figure}

Beyond being able to leverage special NN architecture, from a differential programming perspective, the ``precise understanding of the solution" for SB is also beneficial. At optimality after our training procedure, since we know the optimal control solves a convex optimization problem (\eqref{eqn:bb-1}-\eqref{eqn:bb-2} is not convex in $\rho, u$ as written, but it can be turned into a convex problem in a new set of parameters $\rho,m = \rho u$), if we perturb the target density $\mu(\cdot)$ slightly, one can leverage implicit differentiation to work out a perturbation on the optimal control $u_t(\cdot)$ locally, without differentiating through the solver that is used for training. This is because one can view the operation as a mapping $\mu \mapsto (\rho_t,u_t)$ through the KKT condition. This notion of uniqueness and optimality imply our SB formulation can avoid expensive retraining when deployed in multiple similar instances, in contrast to other diffusion-style samplers \citep{vargas2022denoising, richter2023improved}. We view having an optimality baked into the solution as ideal for blending problem structure into generic off-the-shelf machine learning methods.


\vfill

\section{\uppercase{Various Perspectives on the Schr\"odinger Bridge Problem}}
\label{sec:SB}
Most of these can be traced out in \citep{leonard2014survey,chen2021stochastic}, which has its roots in statistical mechanics \citep{schrodinger1931umkehrung} (and in modern terms, closely related to large-deviation results via Sanov's theorem). For us however, the following perspectives will be more fruitful. 

(1) Mixture of pinned diffusions (weights given by $\pi^*$) can be seen by disintegration of path measure:
\begin{equation}
\label{eqn:path}
D_{KL}(P\Vert Q) = D_{KL}(P_{0T}\Vert Q_{0T})+\int D_{KL} (P^{zx}\Vert Q^{zx})\, dP_{0T}(z,x)\, .
\end{equation}
Since the only constraints are on the two end-points \eqref{eqn:bridge_problem} is therefore equivalent to \eqref{eqn:static}, by choosing $P^{zx}=Q^{zx}$. The optimal solution takes the form $P^*=\pi^*Q^{zx}$, which means one can sample from $(z,x)\sim \pi^*\in \Pi_{\nu,\mu}$, and sample from the bridges conditioned on the end-points at $t=0,T$. 

(2) It has the interpretation of entropy-regularized optimal transport \citep{peyre2019computational} when $f=0$: in terms of static formulation because of \eqref{eqn:path}, if the reference $Q$ is simply the Wiener process (accordingly $r(x|z)\propto e^{-\frac{1}{2T}\|x-z\|^2}$),
\begin{align}
\label{eqn:static}
\pi^*(z,x)&=\arg\min_{\pi_z=\nu, \pi_x=\mu} \; D_{KL}(\pi(z,x)\Vert r(z,x)) \\
&= \arg\min_{\pi_z=\nu, \pi_x=\mu} \; \mathbb{E}_{z\sim \nu}[D_{KL}(\pi(x\vert z)\Vert r(x\vert z))]+ D_{KL}(\nu(z)\Vert r(z))\nonumber\\
&= \arg\min_{\pi_z=\nu, \pi_x=\mu} \int \frac{1}{2}\|x-z\|^2 \pi(z,x)\, dxdz+ T \int \pi(z,x)\log \pi(z,x)\, dxdz  \nonumber\\
&= \arg\min_{\pi_z=\nu, \pi_x=\mu} \int \frac{1}{2}\underbrace{\|x-z\|^2}_{\mathcal{C}(z,x)} \pi(z,x)\, dxdz+ T \underbrace{\int \pi(z,x)\log \frac{\pi(z,x)}{\nu(z)\otimes \mu(x)}\, dxdz}_{D_{KL}(\pi(z,x)\Vert \nu(z)\otimes\mu(x))}\, , \nonumber
\end{align}
where the first term is nothing but the definition of the Wasserstein-2 distance (and optimal transport with quadratic cost in the sense of Kantorovich). The second entropy term favors independent coupling $\nu\otimes \mu$. This is a Lagrangian description of the transport. Different choices of $r(x|z)$ transition will induce different transport costs. The objective above can also be written as 
\[\arg\min_{\pi_z=\nu, \pi_x=\mu}\; D_{KL}\left(\pi(z,x)\,\Vert\, e^{-\mathcal{C}(z,x)/2T}\nu(z)\otimes \mu(x)\right) \, .\] 
In terms of dynamical formulation, via Girsanov's theorem on path measure, with the controlled dynamics as 
\[dX_t = [f_t(X_t)+\sigma u_t(X_t)]\, dt+\sigma dW_t\,,\] \eqref{eqn:bridge_problem} can be reformulated as a \emph{constrained} problem: 
\begin{equation}
\label{eqn:sb_contrain}
P^*=\arg\min_P\, \mathbb{E}_{X\sim\overrightarrow{\mathbb{P}}^{\nu,f+\sigma u}}\left[\frac{1}{2}\int_0^T \|u_t(X_t)\|^2 dt\, \Big\vert\, \overrightarrow{\mathbb{P}}^{\nu,f+\sigma u}_T = \mu\right]\, .
\end{equation}
Or a \emph{regularized} Benamou-Brenier fluid-dynamics \citep{benamou2000computational} analogy of the optimal transport 
\begin{align}
&\inf_{\rho,v} \; \int_{\mathbb{R}^d} \int_0^T \left[\frac{1}{2}\|v_t(X_t)-\bar{v}_t(X_t)\|^2+\frac{\sigma^4}{8}\left\|\nabla \log\frac{\rho_t(X_t)}{\bar{\rho}_t(X_t)}\right\|^2\right] \rho_t(X_t) \, dt dx \label{eqn:continuity-1}\\
\label{eqn:continuity}
&\text{s.t.}\; \frac{\partial \rho_t}{\partial t}+\nabla \cdot (\rho_t v_t)=0, \, \rho_0=\nu, \rho_T=\mu\, .
\end{align}
Above $\bar{v}_t=f_t-\frac{\sigma^2}{2}\nabla \log\bar{\rho}_t$ is the velocity field of the prior process, and we see the penalization results in an additional relative Fisher information term. Since $v^*$ is a velocity field in the continuity equation, it means a deterministic evolution (i.e., ODE) as
\[\dot{X}_t=v_t(X_t),\, X_0\sim \nu\]
will have $X_t\sim \rho_t$, the optimal entropic interpolation flow, which gives an Eulerian viewpoint. 

(3) Optimal control views the problem as steering $\nu$ at $t=0$ to $\mu$ at $t=T$ with minimal control effort. The value function (i.e., optimal cost-to-go)
\begin{equation}
\label{eqn:value_function}
V(x,t):= \min_u\, \mathbb{E}_u\left[\frac{1}{2}\int_t^T \|u_s(X_s)\|^2ds \,\Big\vert\, X_t^u=x, X_T^u \sim \mu\right] 
\end{equation}
with the expectation taken over the stochastic dynamics
\begin{equation}
\label{eqn:control_dynamics}
dX_t^u=[f_t(X_t^u)+\sigma u_t(X_t^u)]dt+\sigma dW_t, \; X_0^u\sim \nu
\end{equation}
should satisfy the Hamilton-Jacobi-Bellman equation via the dynamical programming principle 
\begin{equation}
\label{eqn:hjb}
\frac{\partial V(x,t)}{\partial t}+f_t(x)^\top \nabla V(x,t)+\frac{\sigma^2}{2}\Delta V(x,t)-\frac{\sigma^2}{2}\|\nabla V(x,t)\|^2=0 \,,
\end{equation}
where the optimal control $u_t^*(X_t)=-\sigma\nabla V(X_t,t)$ and gives the unique solution $(\rho_t^u)_{t\geq 0}$ solving 
\begin{equation}
\label{eqn:FPK}
\frac{\partial \rho_t^u}{\partial t}=-\nabla \cdot(\rho_t^u (f_t-\sigma^2\nabla V_t))+\frac{\sigma^2}{2}\Delta \rho_t^u, \quad \rho^u_0\sim \nu,\, \rho^u_T \sim \mu \, .
\end{equation}
The two \emph{coupled} PDEs, Fokker-Planck \eqref{eqn:FPK} and HJB \eqref{eqn:hjb} are also the KKT optimality condition of 
\begin{align}
\label{eqn:bb-1}
&\inf_{\rho,u} \; \int_{\mathbb{R}^d} \int_0^T \frac{1}{2}\|u_t(X_t)\|^2 \rho_t(X_t)\, dt dx\\
\label{eqn:bb-2}
&\text{s.t.}\; \frac{\partial \rho_t}{\partial t}+\nabla \cdot (\rho_t (f_t+\sigma u_t))=\frac{\sigma^2}{2}\Delta \rho_t,\, \rho_0=\nu, \rho_T=\mu
\end{align}
where again the optimal $u^*=-\sigma\nabla V$ is of gradient type. Above the Laplacian is responsible for the diffusion part, and \eqref{eqn:bb-1}-\eqref{eqn:bb-2} is related to \eqref{eqn:continuity-1}-\eqref{eqn:continuity} via a change of variable. One might try to design schemes by forming the Lagrangian for the above \eqref{eqn:bb-1}-\eqref{eqn:bb-2}, and solve the resulting saddle-point problem, but this deviates somewhat from our pathwise narrative. 
\begin{remark}[Langevin]
\label{rmk:langevin_pde}
Compared to the Langevin SDE $dX_t = -\nabla f(X_t) dt + \sqrt{2}dW_t$, which only involves forward-evolving density characterization and reaches equilibrium as $T\rightarrow \infty$, the controlled SDE \eqref{eqn:control_dynamics} is time-inhomogeneous and involves two PDEs \eqref{eqn:hjb}-\eqref{eqn:FPK}. Langevin also has a backward Kolmogorov evolution for the expectation of a function $g$: let $V(x,t)=\mathbb{E}[g(X_T)|X_t=x]$, we have $\partial_t V(x,t)-\nabla f(x)^\top \nabla V(x,t)+\Delta V(x,t)=0$ with $V(x,T)=g(x)$, but it is \emph{de-coupled} from the Fokker-Planck equation $\partial_t\rho_t-\nabla \cdot(\rho_t\nabla f)-\Delta\rho_t=0$ with $\rho_0= \nu$.
\end{remark}

\begin{remark}[Bound on optimal objective]
\label{rmk:bound}
Notice that in fact
\begin{equation}
\label{eqn:bound_temp}
\min_u\; \mathbb{E}_u\left[\int_0^T \frac{1}{2}\|u_t(X_t)\|^2 dt-\log \frac{p_{\text{target}}}{Q_T}(X_T)\right]\geq 0\, ,
\end{equation}
which means we know a lower bound on the minimum of our losses (a) and (d) in Proposition~\ref{prop:loss}. This holds since for the controlled process $P$,
\[D_{KL}(P\Vert Q)=\mathbb{E}_u\left[\int_0^T \frac{1}{2}\|u_t(X_t)\|^2 dt\right]=\mathbb{E}_u\left[\frac{1}{2}\int_0^T \|u_t(X_t)\|^2 dt-\log \frac{p_{\text{target}}}{Q_T}(X_T)\right]+D_{KL}(p_{\text{target}}\Vert Q_T)\, ,\]
which gives the claim by the data processing inequality $D_{KL}(p_{\text{target}}\Vert Q_T)\leq D_{KL}(P\Vert Q)$. A special instance is when $\nu$ is $\delta_0$ a fixed Dirac delta and the reference process is simply Brownian motion $f_t=0$. In this case, the optimal drift $u^*$ is known to be the F\"ollmer drift \citep{tzen2019theoretical,zhang2021path} and can be computed explicitly as (picking $T=1$ for simplicity below) 
\begin{align}
u_t^*(x) &=\arg\min_u\; \mathbb{E}_{X\sim\overrightarrow{\mathbb{P}}^{\nu,\sigma u}}\Big[\int_0^1 \frac{1}{2}\|u_s(X_s)\|^2ds+ 
\log\left(\frac{Q_1}{p_{\text{target}}}\right)(X_1)\Big] \label{eqn:temp_follmer}\\
&=\nabla \log \mathbb{E}_X\left[\frac{d\mu}{dQ_1}(X_1)\,\Big\vert\, X_t=x\right]\label{eqn:Follmer}\\
&=\nabla \log \mathbb{E}_{z\sim\mathcal{N}(0,\sigma^2 I)}\left[\frac{d\mu}{d\mathcal{N}(0,\sigma^2 I)}(x+\sqrt{1-t}z)\right] \nonumber\, ,
\end{align}
where the expectation in \eqref{eqn:Follmer} is taken w.r.t the reference measure -- Wiener process in this case; and it follows from Doob's $h$-transform that the reverse drift $v_t^*(x)=x/t$ is also analytical. This is technically speaking a half-bridge where one can show $P_t^*(X_t) = Q(X_t|X_T)p_{\text{target}}(X_T)$ and equality in \eqref{eqn:bound_temp} holds exactly. \eqref{eqn:temp_follmer} gives an intuitive explanation of the optimally-controlled process, with the first part corresponding to the running cost and the second part to the terminal cost. 
\end{remark}

(4) The non-negative functions $\phi,\psi$ (which are closely related to the dual potentials $f,g$ from \eqref{eqn:factorize_condition-2} below) yield optimal forward/backward drifts: It holds that the optimal curve admits the representation 
\begin{equation}
\label{eqn:factorization}
\log\rho_t=\log\phi_t+\log\psi_t \quad \text{for all } t
\end{equation}
and solving the boundary-coupled linear PDE system on the control
 \begin{equation}
 \label{eqn:sb_system}
 \frac{\partial \phi_t}{\partial t}=-\frac{\sigma^2}{2}\Delta \phi_t-\nabla \phi_t^\top f_t, \quad \frac{\partial \psi_t}{\partial t}= \frac{\sigma^2}{2}\Delta \psi_t-\nabla\cdot(\psi_t f_t) \quad \text{for $\phi_0\psi_0=p_{\text{prior}},\, \phi_T\psi_T=p_{\text{target}}$}
 \end{equation}
 gives two SDEs for the optimal curve in \eqref{eqn:bridge_problem}:
 \begin{align}
 \label{eqn:forward_sb}
 dX_t&=[f_t(X_t)+\sigma^2\nabla \log \phi_t(X_t)]dt+\sigma \overrightarrow{dW_t},\; X_0\sim \nu\\
 \label{eqn:backward_sb}
  dX_t&=[f_t(X_t)-\sigma^2\nabla \log \psi_t(X_t)]dt+\sigma \overleftarrow{dW_t}, \; X_T\sim \mu\, .
  \end{align}
  Note that equations \eqref{eqn:forward_sb}-\eqref{eqn:backward_sb} are time reversals of each other and obey Nelson's identity thanks to \eqref{eqn:factorization}. The transformation \eqref{eqn:factorization}-\eqref{eqn:sb_system} that involves $(\rho_t^*,u_t^*)=(\rho_t^*,-\sigma\nabla V_t)\mapsto (\phi_t,\psi_t)$ is a typical $\log\leftrightarrow\exp$ Hopf-Cole change-of-variable from \eqref{eqn:hjb}-\eqref{eqn:FPK}. These PDE optimality results can be found in \citep{caluya2021wasserstein}. 
  \begin{remark}[Boundary condition]
  In rare cases, such as those from F\"ollmer drift \eqref{eqn:Follmer}, the PDE system \eqref{eqn:sb_system} (or \eqref{eqn:hjb}-\eqref{eqn:FPK}) may decouple and the drift can be expressed analytically. But the boundary conditions of \eqref{eqn:hjb}-\eqref{eqn:FPK} is atypical for control problems, which generally would have \eqref{eqn:FPK} specified at initial time and runs forward, with \eqref{eqn:hjb} specified at the terminal and runs backward. The PDE dynamics \eqref{eqn:sb_system} also takes similar form as those from Remark~\ref{rmk:langevin_pde} but with slightly unusually coupled boundary conditions.
  \end{remark}
  
  (5) Factorization of the optimal coupling: in fact it is always the case that
  \begin{equation}  
  \label{eqn:factorize_condition-2}
  \frac{d\pi^*}{d r}(X_0,X_T)=e^{f(X_0)}e^{g(X_T)}\quad r\text{-}a.s.\,.
  \end{equation}
  Moreover under mild conditions if there exists $\pi,f,g$ for which such decomposition holds and $\pi_0=\nu,\pi_T=\mu$, $\pi$ must be optimal -- such condition \eqref{eqn:factorize_condition-2} is \emph{necessary and sufficient} for characterizing the solution to the SB problem. \eqref{eqn:factorize_condition-2} together with \eqref{eqn:path} give that (which can also be thought of as re-weighting on the path space)
  \begin{equation}
  \label{eqn:factorize_condition-1}
  \frac{d P^*}{d Q}(X_{0:T})=e^{f(X_0)}e^{g(X_T)}\quad Q\text{-}a.s.\, ,
  \end{equation}
  where the $f$ and $g$ obey the Schr\"odinger system
  \begin{equation}
  \label{eqn:sb_system-marginal}
   e^{f(x_0)}\int r(X_0=x_0,X_T)e^{g(X_T)} dX_T =p_{\text{prior}}(x_0)\, ,\quad e^{g(x_T)} \int e^{f(X_0)}r(X_0,X_T=x_T) dX_0 =p_{\text{target}}(x_T)\, .
  \end{equation}
  The $\phi,\psi$ in \eqref{eqn:forward_sb}-\eqref{eqn:backward_sb} can also be expressed as a conditional expectation: 
\[\phi_t(x) \stackrel{(\#)}{=} \int e^{g(X_T)} r(X_T|X_t=x) \, dX_T \stackrel{(*)}{=} \int_{\mathbb{R}^d} \phi_T(X_T) r(X_T|X_t=x) \, dX_T\, ,\]
\[\psi_t(x) \stackrel{(\#)}{=} \int e^{f(X_0)} r(X_t=x,X_0) \, dX_0 \stackrel{(*)}{=} \int_{\mathbb{R}^d} \psi_0(X_0) r(X_t=x|X_0) \, dX_0\]
for $t\in [0,T]$, where $(\#)$ can be verified with \eqref{eqn:sb_system-marginal}, \eqref{eqn:factorization}. And for the second transition $(*)$, one can simply check using the relationship $(\#)$ that

\[\psi_0(X_0)=e^{f(X_0)}r(X_0),\, \phi_T(X_T)=e^{g(X_T)} \]
therefore
\begin{equation}
\label{eqn:factorize_condition}
\frac{d P^*}{d Q}(X_{0:T})=e^{f(X_0)}e^{g(X_T)}=\frac{\psi_0(X_0)\phi_T(X_T)}{r(X_0)}=\frac{\psi_0(X_0)\phi_T(X_T)}{p_{\text{prior}}(X_0)}
\end{equation}
is how the optimal coupling should factorize. 
This representation of $\phi_t,\psi_t$ as the conditional expectation $(*)$ can be seen with the Feynman-Kac formula on \eqref{eqn:sb_system} as well. But in general, the transition kernel of the un-controlled process $r(\cdot)$, or more importantly the terminals $\phi_T,\psi_0$ are not available analytically for solving for $\phi_t,\psi_t$ as $(*)$. 

%
\begin{remark}
Such factorization property \eqref{eqn:factorize_condition-2} is maintained by Sinkhorn algorithm (that updates $f, g$ as iteration proceeds), which is well-known to converge to the optimal coupling eventually.
\end{remark}
  
  (6) Equation \eqref{eqn:forward_sb} means the optimal $v_t^*$ in the continuity equation \eqref{eqn:continuity} is $f_t+\sigma^2\nabla \log \phi_t-\frac{\sigma^2}{2}\nabla \log \rho_t=f_t+\frac{\sigma^2}{2}\nabla \log \frac{\phi_t}{\psi_t}$,
  which gives the probability flow ODE for this dynamics:
  \begin{equation}
  \label{eqn:ode}
  dX_t = f_t(X_t)+\frac{\sigma^2}{2}(\nabla \log\phi_t(X_t)-\nabla \log\psi_t(X_t))\, dt\,,\quad X_0\sim \nu \,,
  \end{equation}
in the sense that the time marginals $\rho_t^{\eqref{eqn:ode}}=\rho_t^{\eqref{eqn:forward_sb}}=\rho_t^{\eqref{eqn:backward_sb}}=P_t^{\eqref{eqn:bridge_problem}}$ all agree. Moreover, using the instantaneous change of variables formula \citep{chen2018neural}, we also have
\begin{equation}
\label{eqn:covf}
\log \rho_T(X_T)=\log \rho_0(X_0)-\int_0^T \nabla \cdot f_t(X_t)\, dt-\frac{\sigma^2}{2}\int_0^T \nabla \cdot (\nabla\log \phi_t(X_t)-\nabla \log\psi_t(X_t))\, dt\, ,
\end{equation}
which will be useful in Proposition \ref{prop:importance_sample} for estimating the normalizing constant. Note that in \eqref{eqn:covf} both the density and the point at which we are evaluating is changing. 
  

\begin{remark}[Minimum control energy]
The stochastic control formulation makes it clear that the trajectory we are trying to recover is a meaningful one in the sense of minimal effort. If one were to switch the order of $P$ and $Q$ in \eqref{eqn:bridge_problem}, the optimal control problem becomes (e.g., for $f=0$)
\begin{align*}
&\inf_u \; \mathbb{E}\left[\int_0^T \frac{1}{2}\|u_t(X_t)\|^2 dt\right]\\
&\text{s.t.}\; dX_t=\sigma u_t(X_t)dt+\sigma dW_t,\, X_0\sim \nu, X_T\sim \mu 
\end{align*}
for the expectation taken over the reference process, instead of the controlled state density $\rho_t^u$, which is not very intuitive. The slightly non-conventional aspect of this control problem is the fixed terminal constraint.
\end{remark}

To briefly summarize, all these different viewpoints explore the deep connections between PDEs (controls) and SDEs (diffusions) in one way or another.

\section{\uppercase{Additional Related Work: Natural attempts for solving SB}}
\label{sec:attempt}
In this section we discuss several natural attempts with the intention of adapting the SB formalism for sampling from un-normalized densities.
\begin{itemize}
\item Iterative proportional fitting (IPF) / Sinkhorn performs iterative projection as
\begin{equation}
\label{eqn:IPF}
P^{(1)}=\arg\min_{Q\in\mathcal{P}(\nu,\cdot)} \; D_{KL}(Q\Vert P^{(0)})=\frac{ P^{(0)}\nu}{P^{(0)}_0}, \quad P^{(0)}=\arg\min_{Q\in\mathcal{P}(\cdot,\mu)} \; D_{KL}(Q\Vert P^{(1)})=\frac{ P^{(1)}\mu}{P^{(1)}_T}
\end{equation}
i.e., one solves half-bridges using drifts learned from the last trajectory transition rollout, and only the end point differ. But since we don't have samples from $\mu$, neither the score function nor the succeeding IPF updates/refinements can be implemented. In fact, the first iteration of the IPF proposal in \citep{de2021diffusion,vargas2021solving} precisely corresponds to the score-based diffusion proposal in \citep{song2020score}. However, there is a connection between IPF and a path space EM implementation formulated in terms of drifts that rely on change in the KL direction: recall under mild assumptions
\[\arg\min_\phi\; D_{KL}(\overrightarrow{\mathbb{P}}^{\nu,f+\sigma^2\nabla \phi}\Vert\overleftarrow{\mathbb{P}}^{\mu,f+\sigma^2\nabla \psi})=\arg\min_\phi\; D_{KL}(\overleftarrow{\mathbb{P}}^{\mu,f+\sigma^2\nabla \psi}\Vert \overrightarrow{\mathbb{P}}^{\nu,f+\sigma^2\nabla \phi})\, .\] 
Using this fact, it is shown in \citep{vargas2023transport} that coordinate descent on the objective $\min_{\phi,\psi}\, D_{KL}(\overrightarrow{\mathbb{P}}^{\nu,f+\sigma^2\nabla \phi}\Vert\overleftarrow{\mathbb{P}}^{\mu,f+\sigma^2\nabla \psi})$, when initializing at $\phi=0$ (i.e., the prior), is a valid strategy for solving \eqref{eqn:bridge_problem}, as it gives the same sequence of path measures as \eqref{eqn:IPF}. 

\begin{lemma}[Optimization of drifts of EM \citep{vargas2023transport}]
\label{lem:em}
The alternating scheme initialized with $\phi_0=0$ converge to $P^*$ (i.e., the coupling at time $t=0,T$ is optimal):
\begin{align}
\label{eqn:drift_em_1}
\psi_{n} &= \arg\min_\psi\; D_{KL}(\overrightarrow{\mathbb{P}}^{\nu,f+\sigma^2\nabla \phi_{n-1}}\Vert\overleftarrow{\mathbb{P}}^{\mu,f+\sigma^2\nabla \psi})\\
\label{eqn:drift_em_2}
\phi_{n} &= \arg\min_\phi\; D_{KL}(\overrightarrow{\mathbb{P}}^{\nu,f+\sigma^2\nabla \phi}\Vert\overleftarrow{\mathbb{P}}^{\mu,f+\sigma^2\nabla \psi_{n}})\, .
\end{align}
Moreover, both updates are implementable assuming samples from $\nu$ is available, which resolves non-uniqueness of the two parameter loss $\min_{\phi,\psi}\, D_{KL}(\overrightarrow{\mathbb{P}}^{\nu,f+\sigma^2\nabla \phi}\Vert\overleftarrow{\mathbb{P}}^{\mu,f+\sigma^2\nabla \psi})$ in an algorithmic manner (fixing one direction of the drift at a time). 
\end{lemma}

\begin{remark}[Benefit of joint training]
Lemma~\ref{lem:em} above also shows that the prior only enters in the first step, therefore as it proceeds, the prior influence tends to be ignored as error accumulates -- this aspect is different from our loss proposals in Proposition~\ref{prop:loss}. Another advantage of our joint minimization procedure compared to an iterative sequential IPF method is that while the equivalence of EM and IPF will no longer hold for neural network models not expressive enough (therefore the convergence does not immediately carry over), one could still estimate the normalizing constant / sample from the target with a sub-optimal control learned from the losses proposed in Proposition~\ref{prop:loss} with e.g., important weights (c.f. \eqref{eqn:iw_statistics}).
\end{remark}

\item Some alternatives to solve SB do not require analytical expression for $\mu$: diffusion mixture matching \citep{peluchetti2023diffusion,shi2024diffusion} tilts the product measure $\nu\otimes \mu$ towards optimality gradually by learning a slightly different term than the score, and alternate between such Markovian and reciprocal projection. Data-driven bridge \citep{pavon2021data} aims at setting up a fixed point recursion on the SB system \eqref{eqn:sb_system-marginal} for finding the optimal $\phi^*,\psi^*$, but both rely on (1) the availability of samples from $\mu$ (as well as $\nu$) to estimate various quantities for implementation; (2) \emph{tractable} bridge distribution/Markov kernel for the reference $Q(\cdot \vert X_0,X_T)$. In fact a different initialization from the reference $Q$ using method in \citep{shi2024diffusion} will reduce to IPF.

\item The work of \citep{caluya2021wasserstein} investigated the case where the reference process has a gradient drift (i.e., $f=-\nabla U$) and reduce the optimal control task to solving a high-dimensional PDE subject to initial-value constraint (c.f. Eqn (33) and (47) therein). However, solving PDEs is largely regarded to be computationally more demanding than simulating SDEs.

\end{itemize}

\newpage
\section{MISSING PROOFS AND CALCULATIONS}
\label{app:proof}
We give the likelihood ratio calculation Lemma~\ref{lem:RN} below, which is closely related to those from \citep{richter2023improved, vargas2023transport} and repeatedly used in our losses. 
\begin{lemma}[Forward / Backward path-space likelihood ratio] 
\label{lem:RN}
Written solely in terms of the drifts, the KL divergence $D_{KL}(\overrightarrow{\mathbb{P}}^{\nu,f+\sigma u}\Vert\overleftarrow{\mathbb{P}}^{\mu,f+\sigma v})$ becomes 
\begin{equation}
\label{eqn:kl_expression}
\mathcal{L}_{KL}(u,v) = \mathbb{E}_{X\sim \overrightarrow{\mathbb{P}}^{\nu,f+\sigma u}}\left[\int_0^T \frac{1}{2}\|u_t(X_t)-v_t(X_t)\|^2 - \nabla \cdot (f_t+\sigma v_t)(X_t) dt+\log\frac{d\nu(X_0)}{d\mu(X_T)}\right]+\log Z\, .
\end{equation}
And the log-variance divergence over the same two path measures can be evaluated to be 
\[\text{Var}_{X\sim \overrightarrow{\mathbb{P}}^{\nu,f+\sigma u}}\left[\int_0^T \frac{1}{2}\|u_t(X_t)-v_t(X_t)\|^2 - \nabla \cdot (f_t+\sigma v_t)(X_t) dt+\log\frac{d\nu(X_0)}{d\mu(X_T)}+\int_0^T (u_t-v_t)(X_t) \, \overrightarrow{dW_t}\right]\, .\]
The reverse one $D_{KL}(\overleftarrow{\mathbb{P}}^{\mu,f+\sigma v}\Vert \overrightarrow{\mathbb{P}}^{\nu,f+\sigma u})$ becomes up to constant
\[\mathbb{E}_{X\sim \overleftarrow{\mathbb{P}}^{\mu,f+\sigma v}}\left[\int_0^T \frac{1}{2}\|v_t(X_t)-u_t(X_t)\|^2 + \nabla \cdot (f_t+\sigma u_t)(X_t) dt+\log\frac{d\mu(X_T)}{d\nu(X_0)}\right]\, .\]
\end{lemma}

\begin{proof}[Proof of Lemma~\ref{lem:RN}]
We give a few equivalent expressions for the Radon-Nikodym derivative. 

\noindent \textbf{Argument due to \citep{vargas2023transport}:} Starting with Proposition 2.2 of \citep{vargas2023transport}, in the general case of reference
\begin{equation}
\label{eqn:ref_sde}
dX_t=\sigma r^+_t(X_t) dt+\sigma \overrightarrow{dW_t}, \; X_0\sim \Gamma_0 \quad\quad dX_t=\sigma r^-_t(X_t) dt+\sigma \overleftarrow{dW_t}, \; X_T\sim \Gamma_T
\end{equation}
where $\overrightarrow{\mathbb{P}}^{\Gamma_0,\sigma r^+}=\overleftarrow{\mathbb{P}}^{\Gamma_T,\sigma r^-}$, it is shown $\overrightarrow{\mathbb{P}}^{\nu,f+\sigma u}$ almost surely, 
\begin{align}
&\log\left(\frac{d\overrightarrow{\mathbb{P}}^{\nu,f+\sigma u}}{d\overleftarrow{\mathbb{P}}^{\mu,f+\sigma v}}\right)(X)=\log (\frac{d\nu}{d\Gamma_0})(X_0)-\log(\frac{d\mu}{d\Gamma_T})(X_T)+\log Z \nonumber\\
&+\frac{1}{\sigma^2}\int_0^T (f_t+\sigma u_t-\sigma r_t^+)(X_t)\left(\overrightarrow{dX_t}-\frac{1}{2}(f_t+\sigma u_t+\sigma r_t^+)(X_t) dt\right)\label{eqn:rn_original-1}\\
&-\frac{1}{\sigma^2}\int_0^T (f_t+\sigma v_t-\sigma r_t^-)(X_t)\left(\overleftarrow{dX_t}-\frac{1}{2}(f_t+\sigma v_t+\sigma r_t^-)(X_t) dt\right)\label{eqn:rn_original-2}\, .
\end{align}
One can use \eqref{eqn:conversion} to convert the backward integral to a forward one with an additional divergence term
\begin{align}
&D_{KL}(\overrightarrow{\mathbb{P}}^{\nu,f+\sigma u}\Vert\overleftarrow{\mathbb{P}}^{\mu,f+\sigma v})=\mathbb{E}_{X\sim \overrightarrow{\mathbb{P}}^{\nu,f+\sigma u}}\left[\log\left(\frac{d\overrightarrow{\mathbb{P}}^{\nu,f+\sigma u}}{d\overleftarrow{\mathbb{P}}^{\mu,f+\sigma v}}\right)(X)\right] \nonumber\\
&=\mathbb{E}_{X\sim \overrightarrow{\mathbb{P}}^{\nu,f+\sigma u}}\bigg[\log (\frac{d\nu}{d\Gamma_0})(X_0)-\log(\frac{d\mu}{d\Gamma_T})(X_T)+\log Z \nonumber\\
&+ \frac{1}{2\sigma^2}\int_0^T (f_t+\sigma u_t-\sigma r_t^+)(X_t)^\top (f_t+\sigma u_t-\sigma r_t^+)(X_t) dt \nonumber\\
&-\frac{1}{\sigma^2}\int_0^T (f_t+\sigma v_t-\sigma r_t^-)(X_t)^\top (\frac{1}{2}f_t+\sigma u_t-\frac{\sigma}{2}v_t-\frac{\sigma}{2}r_t^-)(X_t) dt -\int_0^T \nabla \cdot (f_t+\sigma v_t-\sigma r_t^-)(X_t)dt\bigg] \label{eqn:four_terms}\\
&+\mathbb{E}_{X\sim \overrightarrow{\mathbb{P}}^{\nu,f+\sigma u}}\left[\int_0^T(u_t-r_t^+)(X_t) - (v_t-r_t^-)(X_t)\, \overrightarrow{dW_t}\right] \nonumber
\end{align}
and the last term vanishes because of \eqref{eqn:martingale}. Therefore we see that there are 2 boundary terms, and 3 extra terms corresponding to the forward/backward process. By picking $\gamma^+,\gamma^- =0$, and Lebesgue base measure for $\Gamma_0,\Gamma_T$, we get 
\begin{align}
\label{eqn:kl-divergence-1}
D_{KL}(\overrightarrow{\mathbb{P}}^{\nu,f+\sigma u}\Vert\overleftarrow{\mathbb{P}}^{\mu,f+\sigma v})&=\mathbb{E}_{X\sim \overrightarrow{\mathbb{P}}^{\nu,f+\sigma u}}\left[\int_0^T \frac{1}{2}\|u_t(X_t)-v_t(X_t)\|^2 - \nabla \cdot (f_t+\sigma v_t)(X_t) dt\right]\\
\label{eqn:kl-divergence-2}
&+\mathbb{E}_{X\sim \overrightarrow{\mathbb{P}}^{\nu,f+\sigma u}}\left[\log\frac{d\nu(X_0)}{d\mu(X_T)}+\int_0^T (u_t-v_t)(X_t) \, \overrightarrow{dW_t}\right]+\log Z
\end{align}
which agrees with Proposition 2.3 in \citep{richter2023improved} up to a conventional sign in $v$, and is also the same as the ELBO loss in \cite[Theorem 4]{chen2021likelihood}.

\noindent \textbf{Argument using Remark~\ref{rem:nelson}:} Another way to show \eqref{eqn:kl_expression} is to start with Nelson's identity and apply Girsanov's theorem with \eqref{eqn:forward} and time reversal of \eqref{eqn:backward}. Using the chain rule for the KL between two forward processes we get 
\begin{align}
D_{KL}&(\overrightarrow{\mathbb{P}}^{\nu,f+\sigma u}\Vert\overleftarrow{\mathbb{P}}^{\mu,f+\sigma v})=\mathbb{E}_{\overrightarrow{\mathbb{P}}^{\nu,f+\sigma u}}\left[\log\left(\frac{d\overrightarrow{\mathbb{P}}^{\nu,f+\sigma u}}{d\overleftarrow{\mathbb{P}}^{\mu,f+\sigma v}}\right)(X)\right]\nonumber\\
\label{eqn:gitsanov_forward_sde}
&=\mathbb{E}_{\overrightarrow{\mathbb{P}}^{\nu,f+\sigma u}}\left[ \log\left(\frac{d\nu}{d\overleftarrow{\mathbb{P}}_0^{\mu,f+\sigma v}}\right)(X_0)+\frac{1}{2}\int_0^T\|u_t(X_t)-(v_t+\sigma\nabla\log \rho^{\mu,f+\sigma v}_{t})(X_t)\|^2 dt\right]\, .
\end{align}
To see this is equivalent to \eqref{eqn:kl-divergence-1}-\eqref{eqn:kl-divergence-2}, we use Fokker-Planck on the forward process with drift $v_t+\sigma\nabla\log \rho^{\mu,f+\sigma v}_{t}$ to reach
\begin{align}
&\mathbb{E}_{\overrightarrow{\mathbb{P}}^{\nu,f+\sigma u}}\left[ \log\left(\frac{d\mu}{d\overleftarrow{\mathbb{P}}_0^{\mu,f+\sigma v}}\right)(X)-\log Z\right]=\mathbb{E}_{\overrightarrow{\mathbb{P}}^{\nu,f+\sigma u}}\left[ \log\left(\frac{d\overleftarrow{\mathbb{P}}_T^{\mu,f+\sigma v}}{d\overleftarrow{\mathbb{P}}_0^{\mu,f+\sigma v}}\right)(X)\right]\nonumber\\
\label{eqn:intermediate}
&=\mathbb{E}_{\overrightarrow{\mathbb{P}}^{\nu,f+\sigma u}}\left[\int_0^T -\nabla \cdot (f_t+\sigma v_t)(X_t)+\sigma(u_t-v_t)(X_t)^\top \nabla\log \rho^{\mu,f+\sigma v}_{t}(X_t)-\frac{\sigma^2}{2}\|\nabla\log \rho^{\mu,f+\sigma v}_{t}(X_t)\|^2 dt\right]
\end{align}
as by It\^o's lemma for the process \eqref{eqn:forward} with drift $f+\sigma u$, (denote $\overleftarrow{\mathbb{P}}_t^{\mu,f+\sigma v}$ as $\rho_t$)
\begin{align*}
&\int_0^T \partial_t \log \overleftarrow{\mathbb{P}}_t^{\mu,f+\sigma v}(X_t^{f+\sigma u}) \, dt\\
&=\int_0^T \frac{1}{\rho_t}\left(-\nabla \cdot (\rho_t(f+\sigma v))+\frac{\sigma^2}{2}\Delta\rho_t+\nabla \rho_t^\top (f+\sigma u)\right)dt+\sigma\int_0^T\frac{\nabla \rho_t^\top}{\rho_t}dW_t+\int_0^T \frac{\sigma^2}{2}\Delta \log \rho_t\, dt\\
&= \int_0^T -\frac{\nabla \rho_t^\top}{\rho_t}(f_t+\sigma v_t+\frac{\sigma^2}{2}\nabla \log \rho_t-f_t-\sigma u_t) 
-\nabla\cdot (f_t+\sigma v_t+\frac{\sigma^2}{2}\nabla \log \rho_t) +\frac{\sigma^2}{2}\Delta \log \rho_t\, dt+\int_0^T \sigma \frac{\nabla \rho_t^\top}{\rho_t} \, dW_t\, ,
\end{align*}
which upon simple re-arranging and taking expectation over $X^{f+\sigma u}_t\sim \overrightarrow{\mathbb{P}}_t^{\nu,f+\sigma u}$ give \eqref{eqn:intermediate}. Now adding up the previous two displays \eqref{eqn:gitsanov_forward_sde} and \eqref{eqn:intermediate} using
\[\mathbb{E}_{\overrightarrow{\mathbb{P}}^{\nu,f+\sigma u}}\left[ \log\left(\frac{d\nu}{d\overleftarrow{\mathbb{P}}_0^{\mu,f+\sigma v}}\right)\right]=\mathbb{E}_{\overrightarrow{\mathbb{P}}^{\nu,f+\sigma u}}\left[\log \left(\frac{d\nu}{d\mu}\right)+\log \left(\frac{d\mu}{d\overleftarrow{\mathbb{P}}_0^{\mu,f+\sigma v}}\right) \right]\]
finishes the proof of the expression \eqref{eqn:kl_expression}. The log-variance claim follows easily from this. 

\noindent \textbf{Reverse KL using Remark~\ref{rem:nelson}:} Symmetrically, we can use Girsanov's theorem to compute the reverse KL divergence between two backward processes as: 
\begin{align*}
D_{KL}&(\overleftarrow{\mathbb{P}}^{\mu,f+\sigma v}\Vert \overrightarrow{\mathbb{P}}^{\nu,f+\sigma u})=\mathbb{E}_{\overleftarrow{\mathbb{P}}^{\mu,f+\sigma v}}\left[\log\left(\frac{d\overleftarrow{\mathbb{P}}^{\mu,f+\sigma v}}{d\overrightarrow{\mathbb{P}}^{\nu,f+\sigma u}}\right)(X)\right]\nonumber\\
&=\mathbb{E}_{\overleftarrow{\mathbb{P}}^{\mu,f+\sigma v}}\left[ \log\left(\frac{d\mu}{d\overrightarrow{\mathbb{P}}^{\nu,f+\sigma u}_T}\right)(X_T)+\frac{1}{2}\int_0^T\|v_t(X_t)-(u_t-\sigma\nabla\log \rho_{t}^{\nu,f+\sigma u})(X_t)\|^2 dt-\log Z\right]\, .
\end{align*}
We decompose 
\[\mathbb{E}_{\overleftarrow{\mathbb{P}}^{\mu,f+\sigma v}}\left[ \log\left(\frac{d\mu}{d\overrightarrow{\mathbb{P}}^{\nu,f+\sigma u}_T}\right)\right] = \mathbb{E}_{\overleftarrow{\mathbb{P}}^{\mu,f+\sigma v}}\left[ \log\left(\frac{d\mu}{d\nu}\right)+\log\left(\frac{d\nu}{d\overrightarrow{\mathbb{P}}^{\nu,f+\sigma u}_T}\right)\right]\, .\]
Now note the second term is nothing but $-\int_0^T \partial_t \log \overrightarrow{\mathbb{P}_t}^{\nu,f+\sigma u}(X_t) \, dt$ for $X_t \sim \overleftarrow{\mathbb{P}}^{\mu,f+\sigma v}$, and we compute using Fokker-Planck and backward It\^o's lemma to reach (denote $\overrightarrow{\mathbb{P}_t}^{\nu,f+\sigma u}$ as $\rho_t$)
\begin{align*}
&\int_0^T \partial_t \log \overrightarrow{\mathbb{P}_t}^{\nu,f+\sigma u}(X_t^{f+\sigma v}) \, dt\\
&=\int_0^T \frac{1}{\rho_t}\left(-\nabla \cdot (\rho_t(f+\sigma u))+\frac{\sigma^2}{2}\Delta\rho_t+\nabla \rho_t^\top (f+\sigma v)\right)dt+\sigma\int_0^T\frac{\nabla \rho_t^\top}{\rho_t}\overleftarrow{dW_t}-\int_0^T \frac{\sigma^2}{2}\Delta \log \rho_t\, dt\\
&= \int_0^T -\frac{\nabla \rho_t^\top}{\rho_t}(f_t+\sigma u_t-\frac{\sigma^2}{2}\nabla \log \rho_t-f_t-\sigma v_t) 
-\nabla\cdot (f_t+\sigma u_t-\frac{\sigma^2}{2}\nabla \log \rho_t) -\frac{\sigma^2}{2}\Delta \log \rho_t\, dt+\int_0^T \sigma \frac{\nabla \rho_t^\top}{\rho_t} \, \overleftarrow{dW_t}\\
&= \int_0^T \sigma (v_t-u_t)\nabla \log\rho_t+\frac{\sigma^2}{2}\|\nabla \log\rho_t\|^2-\nabla\cdot(f_t+\sigma u_t)\, dt+\int_0^T \sigma \frac{\nabla \rho_t^\top}{\rho_t} \, \overleftarrow{dW_t}\, ,
\end{align*}
which by using \eqref{eqn:martingale-1} and putting everything together allow us to finish.
\end{proof}


\begin{remark}[Explicit expression]
In Proposition~\ref{prop:loss}, up to constants, $D_{KL}(\overrightarrow{\mathbb{P}}^{\nu,f+\sigma^2 \nabla \phi}\Vert\overleftarrow{\mathbb{P}}^{\mu,f-\sigma^2 \nabla \psi})$ can be written as for $X\sim \overrightarrow{\mathbb{P}}^{\nu,f+\sigma^2 \nabla \phi}$, 
\begin{equation}
\mathbb{E}\Big[\int_0^T \frac{\sigma^2}{2}\|\nabla \phi_t(X_t)+\nabla \psi_t(X_t)\|^2 - 
\nabla \cdot (f_t-\sigma^2 \nabla \psi_t)(X_t) dt 
+\log \frac{\nu(X_0)}{\mu(X_T)}+\sigma\int_0^T (\nabla\phi_t+\nabla \psi_t)(X_t) \, \overrightarrow{dW_t}\Big]\, , \label{eqn:different_kl-1}
\end{equation}
and the log-variance divergence between the same two path measures has $\text{Var}[\cdot]$ in place of $\mathbb{E}[\cdot]$. 
Moreover $D_{KL}(\overrightarrow{\mathbb{P}}^{\nu,f+\sigma^2\nabla \phi_t}\Vert\overleftarrow{\mathbb{P}}^{\mu,f+\sigma^2\nabla \phi_t-\sigma^2\nabla \log\rho_t})$ evaluates to 
\begin{equation}
\mathbb{E}_{\overrightarrow{\mathbb{P}}^{\nu,f+\sigma^2\nabla \phi}}\Big[\log\frac{\nu(X_0)}{\mu(X_T)} + \int_0^T \frac{\sigma^2}{2}\|\nabla \log \rho_t(X_t)\|^2 
+\sigma^2\nabla\cdot (\nabla \log\rho_t- \nabla \phi_t)(X_t)-\nabla \cdot f_t(X_t)\, dt\Big]\, . \label{eqn:different_kl}
\end{equation} 
Both are straightforward consequences of Lemma~\ref{lem:RN}.
\end{remark}

Immediately follows is our main result: Proposition~\ref{prop:loss}, along with Lemma~\ref{lem:sde_for_sb} that sheds a different light on the losses (b) and (c) from Section~\ref{sec:proposal_loss}. 

\begin{proof}[Proof of Proposition~\ref{prop:loss}] 
The first term in losses (a), (b), (d) involving $D_{KL}$ can be written in terms of the (current) controls using Lemma~\ref{lem:RN}, up to constants (c.f. \eqref{eqn:different_kl-1}/\eqref{eqn:different_kl}). One can regularize to enforce optimality while still ensuring the right marginals via 
\[\mathcal{L}(\nabla \phi,\nabla \psi):= D_{KL}(\overrightarrow{\mathbb{P}}^{\nu,\sigma^2\nabla \phi}\Vert\overleftarrow{\mathbb{P}}^{\mu,-\sigma^2\nabla \psi})+\lambda R(\nabla \psi)\quad \text{or} \quad D_{KL}(\overrightarrow{\mathbb{P}}^{\nu,\sigma^2\nabla \phi}\Vert\overleftarrow{\mathbb{P}}^{\mu,-\sigma^2\nabla \psi})+\lambda' R'(\nabla \phi)\]
where $R(\cdot)$ can either utilize the variational formulation recast from \eqref{eqn:sb_contrain} as in (a) or the optimality condition on the control as in (b). In both cases, the first term is a $\rho$ constraint (time reversal consistency with correct terminal marginals), and the second one a control constraint (enforce optimality of trajectory). 

For (b), we give one direction of the argument for $\nabla \psi$ first. Let $\overrightarrow{\mathbb{P}}^{\nu,f}$ denote the path measure associated with $dX_t = f_t(X_t)dt+\sigma dW_t, X_0\sim \nu$, then $\overrightarrow{\mathbb{P}}^{\nu,f}$ almost surely, using Lemma~\ref{lem:RN}, the Radon-Nikodym derivative between $\overrightarrow{\mathbb{P}}^{\nu,f}$ and the controlled backward process is
\[\log\left(\frac{d\overrightarrow{\mathbb{P}}^{\nu,f}}{d\overleftarrow{\mathbb{P}}^{\mu,f-\sigma^2\nabla \psi_t}}\right)(X) = \log\frac{d\nu(X_0)}{d\mu(X_T)}+\int_0^T \frac{\sigma^2}{2}\|\nabla \psi_t\|^2- \nabla \cdot(f-\sigma^2\nabla \psi_t)dt+\sigma\int_0^T\nabla \psi_t^\top\, dW_t+\log Z\, .\]
In the case when the variance (taken along the prior $X\sim\overrightarrow{\mathbb{P}}^{\nu,f}$)
\[\text{Var}\left(\psi_T(X_T)-\psi_0(X_0) - \int_0^T \frac{\sigma^2}{2}\|\nabla \psi_t\|^2+\nabla \cdot f-\sigma^2\Delta \psi_t(X_t)dt-\int_0^T \sigma\nabla \psi_t^\top\, dW_t\right)=0\, ,\] 
it implies that the random quantity is almost surely a constant independent of the realization, and
\[\log\left(\frac{d\overrightarrow{\mathbb{P}}^{\nu,f}}{d\overleftarrow{\mathbb{P}}^{\mu,f-\sigma^2\nabla \psi_t}}\right)(X)=\underbrace{\log \nu(X_0)-\psi_0(X_0)}_{-f(X_0)}+\underbrace{\psi_T(X_T)-\log \mu(X_T)+\log Z}_{-g(X_T)}\, .\]
Recall the terminal constraint $\overrightarrow{\mathbb{P}}^{\nu,f+\sigma^2\nabla \phi_t}_0=\overleftarrow{\mathbb{P}}^{\mu,f-\sigma^2\nabla \psi_t}_0=\nu$ and $\overrightarrow{\mathbb{P}}^{\nu,f+\sigma^2\nabla \phi_t}_T=\overleftarrow{\mathbb{P}}^{\mu,f-\sigma^2\nabla \psi_t}_T=\mu$ are imposed by the first KL term, therefore using Nelson's identity $-\nabla\psi_0(X_0)=\nabla\phi_0(X_0)-\nabla \log \nu(X_0)$ and $-\nabla\psi_T(X_T)=\nabla\phi_T(X_T)-\nabla\log p_{\text{target}}(X_T)$, from which we can deduce that the factorization property \eqref{eqn:factorize_condition} holds and concludes that $\overleftarrow{\mathbb{P}}^{\mu,f-\sigma^2\nabla \psi_t}$ must be the unique solution to the SB problem. The other direction on $\nabla \phi$ is largely similar, and one can show using Girsanov's theorem and the variance condition that along the prior $X\sim\overrightarrow{\mathbb{P}}^{\nu,f}$,
\begin{align*}
\log \left(\frac{d\overrightarrow{\mathbb{P}}^{\nu,f+\sigma^2\nabla \phi_t}}{d\overrightarrow{\mathbb{P}}^{\nu,f}}\right)(X) &=\int_0^T -\frac{\sigma^2}{2}\|\nabla \phi_t\|^2 dt + \int_0^T \sigma\nabla \phi_t^\top \,dW_t = \underbrace{\phi_T(X_T)}_{g(X_T)}\underbrace{-\phi_0(X_0)}_{f(X_0)}\, .
\end{align*}
This allows us to conclude that the factorization characterization \eqref{eqn:factorize_condition} holds in the same way. Note that evaluating the variance regularizer only requires simulating from the reference process \eqref{eqn:Q_sde}. 

In (c), the first two parts enforce
\[\phi_T(X_T)+\psi_T(X_T)=\log p_{\text{target}}(X_T) \quad\text{and} \quad\phi_0(X_0)+\psi_0(X_0)=\log p_{\text{prior}}(X_0)=\log\nu(X_0)\, .\]
Now for any $\phi_t$, the likelihood ratio along $dX_t=(f+\sigma^2\nabla \phi_t)dt+\sigma dW_t, X_0\sim \nu$ is 
\begin{equation}
\label{eqn:temp1}
\log \left(\frac{d\overrightarrow{\mathbb{P}}^{\nu,f+\sigma^2\nabla \phi_t}}{d\overrightarrow{\mathbb{P}}^{\nu,f}}\right)(X) = \int_0^T \frac{\sigma^2}{2}\|\nabla \phi_t\|^2dt+\int_0^T \sigma\nabla \phi_t^\top\, dW_t\, ,
\end{equation}
which according to \eqref{eqn:factorize_condition}, has to be equal to 
\[\psi_0(X_0)+\phi_T(X_T)-\log\nu(X_0)=-\phi_0(X_0)+\phi_T(X_T) \,\;\;  \mathbb{P}^{\nu,f} \, a.s. \Rightarrow \mathbb{P}^{\nu,f+\sigma^2\nabla \phi}\, a.s.\]
for $\nabla \phi$ to be optimal, justifying
\begin{align*}
&\text{Var}_{X\sim \overrightarrow{\mathbb{P}}^{\nu,f+\sigma^2 \nabla \phi}}\left(\phi_T(X_T)-\phi_0(X_0)-\frac{\sigma^2}{2}\int_0^T \|\nabla \phi_t\|^2(X_t)\, dt-\sigma\int_0^T \nabla \phi_t(X_t)^\top\, dW_t \right)\, . 
\end{align*}
In a similar spirit, using \eqref{eqn:rn_original-1}-\eqref{eqn:rn_original-2} and \eqref{eqn:conversion}, for any $\psi_t$, along the same process $X\sim \overrightarrow{\mathbb{P}}^{\nu,f+\sigma^2\nabla \phi_t}$,
\begin{align}
\log\left(\frac{d\overrightarrow{\mathbb{P}}^{\nu,f}}{d\overleftarrow{\mathbb{P}}^{\mu,f-\sigma^2\nabla \psi_t}}\right)(X)&=\int_0^T \sigma^2\nabla \psi_t^\top\nabla \phi_t +\frac{\sigma^2}{2}\|\nabla \psi_t\|^2-\nabla \cdot (f_t-\sigma^2\nabla \psi_t)dt+\int_0^T \sigma \nabla \psi_t^\top\, dW_t \nonumber\\
&+\log \left(\frac{d\nu(X_0)}{d\mu(X_T)}\right) +\log Z\, , \label{eqn:temp2}
\end{align}
which again using \eqref{eqn:factorize_condition}, has to be equal to \[-\psi_0(X_0)-\phi_T(X_T)+\log \nu(X_0)=\log \nu(X_0)-\psi_0(X_0)+\psi_T(X_T)-\log p_{\text{target}}(X_T)\]
for $\nabla \psi$ to be optimal, yielding the claimed variance regularizer. Additionally, in order to verify the terminal constraint at $0,T$ along $\overrightarrow{\mathbb{P}}^{\nu,f+\sigma^2\nabla \phi_t}$, it suffices to sum up \eqref{eqn:temp1}-\eqref{eqn:temp2}, which due to the first 2 terms of the loss, gives 
\[-\phi_0(X_0)+\phi_T(X_T)-\psi_0(X_0)-\log p_{\text{target}}(X_T)+\psi_T(X_T)+\log \nu(X_0)=0.\]
This necessarily imposes the time-reversal consistency $\overrightarrow{\mathbb{P}}^{\nu,f+\sigma^2\nabla \phi_t}/\overleftarrow{\mathbb{P}}^{\mu,f-\sigma^2\nabla \psi_t} = 1$ a.s. for the LHS.

In (d), we use $\nabla \phi_t$ and $\nabla \log\rho_t$ as optimization variables instead of the two drifts, and it follows from the dynamical formulation \eqref{eqn:bb-1}-\eqref{eqn:bb-2}. 
 The first part establishes a particular relationship between the two variables (namely $\nabla\phi_t$ traces out a curve of measures $\rho_t$), and the second part enforces optimality among all curves transporting between $\nu$ and $\mu$. 
\end{proof}

It is important to note that the dynamics for $\phi_t,\psi_t$ below are adapted to the same filtration generated by the Brownian motion corresponding to $X_t\sim \overrightarrow{\mathbb{P}}$, i.e., $\phi_t,\psi_t$ are interpreted as functions of $X_t,t$. They are not time-reversed SDE, but rather terminal-constrained SDEs.
\begin{lemma}[SDE correspondence to SB optimality] 
\label{lem:sde_for_sb}
We have for the optimal forward drift $\nabla \phi_t$ and $X\sim\overrightarrow{\mathbb{P}}^{\nu,f}$ as in \eqref{eqn:Q_sde},
\[d \phi_t(X_t) = - \frac{\sigma^2}{2}\|\nabla \phi_t(X_t)\|^2 dt+\sigma \nabla \phi_t(X_t)^\top \,\overrightarrow{dW_t}\, ,\]
analogously for the optimal backward drift $-\nabla\psi_t$, along $X\sim\overrightarrow{\mathbb{P}}^{\nu,f}$, 
\[d\psi_t(X_t)=\Big[\frac{\sigma^2}{2}\|\nabla \psi_t\|^2-\nabla \cdot f_t+\sigma^2\Delta \psi_t\Big](X_t) dt+\sigma\nabla \psi_t(X_t)^\top\, \overrightarrow{dW_t}\, .\]
Moreover, along the controlled forward dynamics $X\sim \overrightarrow{\mathbb{P}}^{\nu,f+\sigma^2\nabla \phi}$, the optimal control $\phi,\psi$ satisfy
\begin{align}
&d\phi_t(X_t) = \frac{\sigma^2}{2}\|\nabla \phi_t(X_t)\|^2 dt+\sigma \nabla \phi_t(X_t)^\top \,\overrightarrow{dW_t} \label{eqn:fbsde_control_1}\, ,\\
&d\psi_t(X_t)=\Big[\frac{\sigma^2}{2}\|\nabla \psi_t\|^2+\nabla \cdot (\sigma^2\nabla \psi_t-f_t)+\sigma^2\nabla \phi_t^\top\nabla \psi_t\Big](X_t)dt+\sigma\nabla \psi_t(X_t)^\top \overrightarrow{dW_t} \label{eqn:fbsde_control_2} \, .
\end{align}
In the above, $\nabla\phi_t,\nabla\psi_t$ refer to the optimal forward / backward drift in the SDE
\[dX_t=[f_t(X_t)+\sigma^2\nabla \phi_t(X_t)] dt+\sigma \overrightarrow{dW_t},\; X_0\sim \nu\,,\]
\[dX_t=[f_t(X_t)-\sigma^2\nabla \psi_t(X_t)]dt+\sigma \overleftarrow{dW_t}, \; X_T\sim \mu\, ,\]
and $\phi_T(X_T)+\psi_T(X_T)=\log p_{\text{target}}(X_T), \phi_0(X_0)+\psi_0(X_0)=\log \nu(X_0)$ for $X\sim \overrightarrow{\mathbb{P}}^{\nu,f+\sigma^2\nabla \phi}$.
\end{lemma}
\begin{proof}[Proof of Lemma~\ref{lem:sde_for_sb}]
Using It\^o's lemma, we have along the reference SDE $dX_t=f_t(X_t)dt+\sigma dW_t, X_0\sim \nu$, 
\[d\phi_t=\left[\frac{\partial \phi_t}{\partial t}+\nabla \phi_t^\top f_t+\frac{\sigma^2}{2}\Delta \phi_t\right] dt+\sigma\nabla \phi_t^\top \overrightarrow{dW_t}\, .\]
Now deducing from \eqref{eqn:sb_system}, since the optimal $\phi_t$ solves the PDE for all $(t,x)\in[0,T]\times \mathbb{R}^d$
\[\partial_t \phi_t=-f_t^\top\nabla \phi_t-\frac{\sigma^2}{2}\Delta\phi_t-\frac{\sigma^2}{2}\|\nabla \phi_t\|^2\, ,\]
substituting the last display into the previous one gives the result. Analogously, along the same reference process with It\^o's lemma, 
\[d\psi_t = \left[\frac{\partial \psi_t}{\partial t}+\nabla \psi_t^\top f_t+\frac{\sigma^2}{2}\Delta \psi_t\right] dt+\sigma\nabla \psi_t^\top \overrightarrow{dW_t}\]
and using the fact \eqref{eqn:sb_system} that the optimal $\psi_t$ solves for all $(t,x)\in[0,T]\times \mathbb{R}^d$
\[\partial_t \psi_t = -\nabla\psi_t^\top f_t-\nabla \cdot f_t+\frac{\sigma^2}{2}\left(\|\nabla \psi_t\|^2+\Delta\psi_t\right)\]
and plugging into the previous display yields the claim. In both of the PDE derivations above, we used the fact that for any $g\colon \mathbb{R}^d\rightarrow \mathbb{R}$, 
\begin{equation}
\label{eqn:important_identity}
\frac{1}{g}\nabla^2 g=\nabla^2\log g+\frac{1}{g^2}\nabla g\nabla g^\top\Rightarrow \frac{1}{g}\Delta g=\Delta \log g+\|\nabla \log g\|^2
\end{equation}
by taking trace on both sides. 

The second part of the lemma statement, where $X_t$ evolves along the optimally controlled SDE, follows from \cite[Theorem 3]{chen2021likelihood} up to a change of variable. Notice the sign change and the absence of the cross term in the dynamics for $\phi_t,\psi_t$ when $X\sim\overrightarrow{\mathbb{P}}^{\nu,f}$ vs. $X\sim \overrightarrow{\mathbb{P}}^{\nu,f+\sigma^2\nabla \phi}$.
\end{proof}

\begin{remark}[Equivalence between SDE, PDE, Path measure]
\label{rmk:equivalence}
Adding \eqref{eqn:fbsde_control_1}-\eqref{eqn:fbsde_control_2} up and integrating over time, we get that along the \emph{optimally controlled} forward trajectory, 
\[\log p_{\text{target}}(X_T)-\log\nu(X_0)=\int_0^T \frac{\sigma^2}{2}\|\nabla \phi_t(X_t)+\nabla \psi_t(X_t)\|^2 + 
\nabla \cdot (\sigma^2 \nabla \psi_t-f_t)(X_t) dt +\sigma\int_0^T (\nabla\phi_t+\nabla \psi_t)(X_t) \, \overrightarrow{dW_t},\]
exactly matching the KL objective \eqref{eqn:different_kl-1} $D_{KL}(\overrightarrow{\mathbb{P}}^{\nu,f+\sigma^2 \nabla \phi}\Vert\overleftarrow{\mathbb{P}}^{\mu,f-\sigma^2 \nabla \psi})=0$. This implies that the optimally controlled dynamics is one solution that satisfy the $\nu-\mu$ marginal. The summation of $\phi_t,\psi_t$, however, introduces ambiguity/non-uniqueness as any $+C$ shift in one of $\phi,\psi$ that's cancelled by $-C$ in another will result in the same integral equality for $\partial_t \log\rho_t(X_t)$. What we really want is therefore conditions on $\log \left(\frac{d\overrightarrow{\mathbb{P}}^{\nu,f+\sigma^2\nabla \phi_t}}{d\overrightarrow{\mathbb{P}}^{\nu,f}}\right)\, \text{and} \,\log\left(\frac{d\overrightarrow{\mathbb{P}}^{\nu,f}}{d\overleftarrow{\mathbb{P}}^{\mu,f-\sigma^2\nabla \psi_t}}\right)$ separately instead of $\log\left(\frac{d\overrightarrow{\mathbb{P}}^{\nu,f+\sigma^2\nabla \phi}}{d\overleftarrow{\mathbb{P}}^{\mu,f-\sigma^2\nabla \psi}}\right)$
only. An equivalent PDE representation of $\partial_t \log\rho_t(X_t)$ can be written as 
\begin{align}
\partial_t \log\rho_t(X_t)&= \frac{1}{\rho_t}(\partial_t \rho_t+\nabla \rho_t^\top \dot{X_t})\nonumber\\
&=\frac{1}{\rho_t}[-\nabla \cdot (\rho_t(f_t+\sigma^2\nabla \phi_t))+\frac{\sigma^2}{2}\Delta\rho_t]+\frac{1}{\rho_t}\nabla \rho_t^\top (f_t+\sigma^2\nabla \phi_t-\frac{\sigma^2}{2}\nabla \log \rho_t)\nonumber\\
&=-\sigma^2\Delta\phi_t-\nabla \cdot f_t+\frac{\sigma^2}{2}\frac{1}{\rho_t}\Delta\rho_t-\frac{\sigma^2}{2}\|\nabla \log\rho_t\|^2=-\sigma^2\Delta\phi_t-\nabla \cdot f_t+\frac{\sigma^2}{2}\Delta\log\rho_t\, . \label{eqn:partial_log_rho} 
\end{align}
\end{remark}

Proposition~\ref{prop:importance_sample} and Lemma~\ref{lem:discretization} are stated in Section~\ref{sec:implement}, whose proof we give below.

\begin{proof}[Proof of Proposition~\ref{prop:importance_sample}]
The optimal $\nabla \phi_t,\nabla \psi_t$ allow us to calculate $\log Z$ for $p_{\text{target}}$ using \eqref{eqn:different_kl-1} as (can be used without expectation, or with expectation and ignore the last zero-mean martingale term) 
\begin{align*}
\log Z&=\mathbb{E}_{\overrightarrow{\mathbb{P}}^{\nu,f+\sigma^2 \nabla \phi}}\left[-\frac{\sigma^2}{2}\int_0^T \|\nabla \phi_t(X_t)+\nabla \psi_t(X_t)\|^2 dt+\int_0^T \nabla \cdot (f_t(X_t)-\sigma^2 \nabla \psi_t(X_t)) dt
-\log\frac{d\nu(X_0)}{d\mu(X_T)}\right]\\
&+\mathbb{E}_{\overrightarrow{\mathbb{P}}^{\nu,f+\sigma^2 \nabla \phi}}\left[-\sigma \int_0^T(\nabla \phi_t + \nabla \psi_t)(X_t)\, \overrightarrow{dW_t}\right] =: \mathbb{E}[-S]\,.
\end{align*}
Alternatively using \eqref{eqn:covf}, with the optimal $\nabla \phi^*,\nabla \psi^*$, since $Z$ is independent of $X_T$,
\begin{equation}
\label{eqn:jarzynski}
-\log Z = \log \nu(X_0)-\frac{\sigma^2}{2}\int_0^T\nabla\cdot (\nabla \log\phi_t^*-\nabla \log\psi_t^*)(X_t)\, dt-\int_0^T \nabla \cdot f_t(X_t) dt -\log \mu(X_T)\, ,
\end{equation}
in which case the estimator is exact with $X_t$ following \eqref{eqn:ode}. 
 
In general for imperfect control, since $D_{KL}(\overrightarrow{\mathbb{P}}^{\nu,f+\sigma^2 \nabla \phi}\Vert\overleftarrow{\mathbb{P}}^{\mu,f-\sigma^2\nabla \psi})>0$, $\log Z$ will only be lower bounded by $\mathbb{E}[-S]$. Using Lemma~\ref{lem:RN} however,
\begin{align} 
&1 = \mathbb{E}_{\overrightarrow{\mathbb{P}}^{\nu,f+\sigma^2\nabla \phi}}\left[\left( \frac{d\overrightarrow{\mathbb{P}}^{\nu,f+\sigma^2\nabla \phi}}{d\overleftarrow{\mathbb{P}}^{\mu,f-\sigma^2\nabla \psi}}\right)^{-1}\right]\nonumber\\
&= \mathbb{E}_{\overrightarrow{\mathbb{P}}^{\nu,f+\sigma^2\nabla \phi}}\left[\exp\left(-\frac{\sigma^2}{2}\int_0^T \|\nabla \phi_t+\nabla \psi_t\|^2+\nabla \cdot (f_t-\sigma^2\nabla \psi_t)dt-\sigma\int_0^T (\nabla \phi_t+\nabla \psi_t) dW_t-\log\frac{\nu}{\mu}\right)\frac{1}{Z}\right] \label{eqn:Z_estimate}\\
&=:\mathbb{E}[\exp(-S')/Z]\, , \nonumber
\end{align}
giving $Z=\mathbb{E}_{\overrightarrow{\mathbb{P}}^{\nu,f+\sigma^2\nabla \phi}}[\exp(-S')]$ as unbiased estimator using any (potentially sub-optimal) controls $\nabla \phi,\nabla \psi$. 

For the importance sampling, we use path weights suggested by the terminal requirement that $X_T^{\phi^*}\sim \mu$, therefore using \eqref{eqn:partial_log_rho} the path weight becomes 
\begin{align*}
w^\phi(X^\phi_T)&=\frac{d\mu}{d\mathbb{P}_{X_T^\phi}}(X_T^\phi)= \frac{d\mu (X_T^\phi)}{d\nu (X_0^\phi)}\frac{d\nu (X_0^\phi)}{d\mathbb{P}_{X_T^\phi}(X_T^\phi)}\\
&= \frac{d\mu (X_T^\phi)}{d\nu (X_0^\phi)}\exp\left(-\int_0^T \partial_t \log \rho^\phi_t(X_t^\phi) \, dt\right)\\
&=  \frac{d\mu (X_T^\phi)}{d\nu (X_0^\phi)}\exp\left(\int_0^T \sigma^2\Delta\phi_t-\frac{\sigma^2}{2}\Delta\log\rho_t^\phi+\nabla \cdot f_t\, dt\right)\, .
\end{align*}
Indeed, this choice guarantees $\mathbb{E}_\phi[g(X_T^\phi)w^\phi(X_T^\phi)]=\int g(x)\mu(x) dx$ for any function $g$ on the terminal variable $X_T$ generated with the suboptimal control $\phi$. It remains to normalize the weights (i.e., account for the constant $Z$) via dividing by $\mathbb{E}_\phi[w^\phi(X_T^\phi)]=Z$. In the case of perfect controls, all samples will have equal weight.
\end{proof}


\begin{remark}[optimal control $\Leftrightarrow$ optimal estimator]
\label{rem:jarzynski} 
In general one always has the importance sampling identity 
\[1 = \mathbb{E}_{\overrightarrow{\mathbb{P}}^{\nu,f+\sigma^2\nabla \phi}}\left[\left( \frac{d\overrightarrow{\mathbb{P}}^{\nu,f+\sigma^2\nabla \phi}}{d\overleftarrow{\mathbb{P}}^{\mu,f-\sigma^2\nabla \psi}}\right)^{-1}\right]\, .\]
With the optimal controls $\phi^*,\psi^*$, we see that since inside the square bracket of \eqref{eqn:Z_estimate}, we have $\exp(0)=1$
holds deterministically (although the path is random), the $Z$-estimator is optimal when optimally controlled in the sense that it's unbiased and has zero variance (it is almost surely a constant and equality holds not only in expectation). This property is desirable for the Monte-Carlo estimates that we employ. Such optimal control $\Leftrightarrow$ optimal sampling equivalence also appears in \citep{thijssen2015path} for path integral control problem that involves a \emph{single control}. This choice of $\phi^*,\psi^*$ also admits the interpretation of minimizing the $\chi^2$-divergence between path measures:
\[\chi^2(\overleftarrow{\mathbb{P}}^{\mu,f-\sigma^2\nabla \psi}\Vert d\overrightarrow{\mathbb{P}}^{\nu,f+\sigma^2\nabla \phi}) = \text{Var}_{\overrightarrow{\mathbb{P}}^{\nu,f+\sigma^2\nabla \phi}} \left[ \frac{\overleftarrow{\mathbb{P}}^{\mu,f-\sigma^2\nabla \psi}}{\overrightarrow{\mathbb{P}}^{\nu,f+\sigma^2\nabla \phi}}\right]\, .\]
\end{remark}

\begin{proof}[Proof of Lemma~\ref{lem:discretization}]
We obtain the forward trajectory 
\begin{equation}
\label{eqn:new_forward}
X_{k+1}^i=X_k^i+(f_k(X_k^i)+\sigma^2 \nabla \phi_k(X_k^i))h+\sigma\sqrt{h} \cdot z_k^i, \; z_k^i\sim\mathcal{N}(0,I) 
\end{equation}
with Euler-Maruyama for each of the $i\in[N]$ samples. Rewriting the term \eqref{eqn:temp2} using \eqref{eqn:conversion}, \eqref{eqn:rn_original-1}-\eqref{eqn:rn_original-2} and proceeding with the approximation \eqref{eqn:approx}-\eqref{eqn:approx-1} give
\begin{align}
&\log\frac{\nu(X_0)}{p_{\text{target}}(X_T)}+\frac{1}{\sigma^2}\int_0^T f_t(X_t)\overrightarrow{dX_t}-\frac{1}{2\sigma^2}\int_0^T \|f_t(X_t)\|^2 dt-\frac{1}{\sigma^2}\int_0^T (f_t-\sigma^2\nabla \psi_t)(X_t)\overleftarrow{dX_t}+\frac{1}{2\sigma^2}\int_0^T \|(f_t-\sigma^2\nabla \psi_t)(X_t)\|^2 dt\nonumber\\
&\approx\frac{1}{\sigma^2}\sum_{k=0}^{K-1} f_k(X_k)^\top (X_{k+1}-X_k)-\frac{1}{2\sigma^2}\|f_k(X_k)\|^2h \nonumber\\
&-\frac{1}{\sigma^2} (f_{k+1}(X_{k+1})-\sigma^2\nabla \psi_{k+1}(X_{k+1}))(X_{k+1}-X_k)+\frac{1}{2\sigma^2}\|f_{k+1}(X_{k+1})-\sigma^2\nabla \psi_{k+1}(X_{k+1})\|^2h +\log\frac{\nu(X_0)}{p_{\text{target}}(X_T)}\nonumber\\
&= \sum_{k=0}^{K-1} \frac{1}{2\sigma^2h} \|X_{k}-X_{k+1}+(f_{k+1}-\sigma^2\nabla \psi_{k+1})(X_{k+1})h\|^2 \nonumber - \sum_{k=0}^{K-1} \frac{1}{2\sigma^2h} \|X_{k+1}-X_k-f_{k}(X_k)h\|^2+\log\frac{\nu(X_0)}{p_{\text{target}} (X_T)}\nonumber \\
&\overset{!}{=}  \log \nu(X_0)-\psi_0(X_0)+\psi_T(X_T)-\log p_{\text{target}} (X_T) \nonumber
\end{align}
which yields the estimator for the variance regularizer on $\psi$ when $(X_k)_k$ follows \eqref{eqn:new_forward}. Note that compared to a naive discretization of the RHS of \eqref{eqn:temp2}, we are able to avoid any divergence and stochastic terms.


It is also possible to avoid the divergence term in \eqref{eqn:Z_estimate} by leveraging similar ideas. Direct computation using \eqref{eqn:rn_original-1}-\eqref{eqn:rn_original-2} on $\frac{d\overleftarrow{\mathbb{P}}^{\mu,f-\sigma^2\nabla \psi}}{d\overrightarrow{\mathbb{P}}^{\nu,f+\sigma^2\nabla \phi}}$ tells us the normalizing constant estimator
\[Z\approx \frac{1}{N}\sum_{i=1}^N \exp\left(\log\frac{\mu}{\nu}-\sum_{k=0}^{K-1} \frac{1}{2\sigma^2h}\|X_k^i-X_{k+1}^i+(f_{k+1}-\sigma^2\nabla \psi_{k+1})(X_{k+1}^i)h\|^2+\frac{1}{2}\|z_k
^i\|^2\right)\]
for the same Euler-Maruyama trajectory \eqref{eqn:new_forward}, where we used that the additional term
\[\sum_{k=0}^{K-1} \frac{1}{2\sigma^2h} \|X_{k+1}-X_k-(f_{k}+\sigma^2\nabla \phi_k)(X_k)h\|^2=\frac{1}{2}\sum_{k=0}^{K-1}\|z_k\|^2\]
from the update \eqref{eqn:new_forward}.  
\end{proof} 

\section{MISSING AUXILIARY RESULTS}
\begin{remark}[Backward drift]
\label{rmk:backward_drift}
For PINN \eqref{pinn:loss}, a regularizer on the backward drift involving $\nabla \psi$ is also possible and will be:
\begin{equation}
\eqref{eqn:log_divergence} +\lambda \cdot \frac{h}{n}\sum_{i=1}^n \sum_{k=0}^{K} \Big|\partial_t \psi(x_k^i,kh)-\frac{\sigma^2}{2}\left[\Delta \psi(x_k^i,kh)-
\|\nabla \psi(x_k^i,kh)\|^2\right]\Big|\label{eqn:backward_pinn}\, .
\end{equation}
And similarly for variance loss \eqref{var:loss}, which will read as 
\begin{align}
\eqref{eqn:log_divergence}&+\frac{\lambda}{K+1} \cdot \text{Var}_n \Big[\psi(y_{K+1}^i,(K+1)h)-\psi(y_0^i,0)+\sum_{k=0}^K \frac{1}{2\sigma^2h}\Big(\|y_{k+1}^i-y_k^i\|^2- \|y_{k}^i-y_{k+1}^i-\sigma^2 h \nabla \psi(y_{k+1}^i,(k+1)h)\|^2\Big)\Big] \label{eqn:backward_var}
\end{align}
and TD loss \eqref{eqn:TD} becomes
\begin{align}
\eqref{eqn:log_divergence}&+\lambda\cdot \frac{1}{n}\sum_{i=1}^n \sum_{k=0}^K h\cdot \Big\vert \psi(y_{k+1}^i,(k+1)h)-\psi(y_{k}^i,kh)-\frac{\sigma^2 h}{2} \|\nabla \psi(y_{k}^i,kh)\|^2 -\sigma^2 \Delta \psi(y_k^i,kh) -\sigma\sqrt{h} \nabla\psi(y_{k}^i,kh)^\top Z_k^i\Big\vert\, . \label{eqn:backward_td}
\end{align}
However, these should witness similar behavior as the regularizer on the forward control $\nabla \phi$ that we experimented with (using the same trajectories). Note crucially there is no divergence term in \eqref{eqn:backward_var} in contrast to \eqref{eqn:backward_td}, \eqref{eqn:backward_pinn}.
\end{remark}


Lastly we mention a word about contraction for method \eqref{eqn:score_matching_loss}.

\begin{remark}[Contraction]
\label{rmk:contract}
In Section~\ref{sec:related} we claimed that for two processes $(p_t)_t,(q_t)_t$ with different initialization, but same drift $\sigma v+\sigma^2 s$ where $s$ is the score function, $\partial_t\, D_{KL}(p_t\Vert q_t)=-\sigma^2/2\cdot \mathbb{E}_{p_t}[\|\nabla \log\frac{p_t}{q_t}\|^2]\leq 0$ contracts towards each other, when we run the generative process 
\[dX_t=[\sigma v_t(X_t)+\sigma^2 s_t(X_t)]\, dt+\sigma \overrightarrow{dW_t},\; X_0\sim \nu\neq \overleftarrow{\mathbb{P}}_0^{\mu,\sigma v}\, ,\]
although the actual rate may be slow. We give this short calculation here for completeness:
\begin{align*}
\partial_t\, &D_{KL}(p_t\Vert q_t)\\
&=\int (\partial_t p_t(x))\log \frac{p_t(x)}{q_t(x)} \, dx+\int q_t(x)\left[\frac{\partial_t p_t(x)}{q_t(x)}-\frac{p_t(x)\partial_t q_t(x)}{q_t^2(x)}\right]\,dx\\
&= \int p_t\left[\sigma v_t+\sigma^2 s_t-\frac{\sigma^2}{2}\nabla \log p_t(x)\right]^\top\nabla \log \frac{p_t(x)}{q_t(x)} dx -\int \frac{p_t(x)}{q_t(x)}\partial_t q_t(x) dx\\
&= \int p_t\left[\sigma v_t+\sigma^2 s_t-\frac{\sigma^2}{2} \nabla \log p_t(x)\right]^\top\nabla \log \frac{p_t(x)}{q_t(x)} dx -\int \nabla \frac{p_t(x)}{q_t(x)}\left[\sigma v_t+\sigma^2 s_t-\frac{\sigma^2}{2}\nabla \log q_t(x)\right] q_t(x) dx\\
&= \int \left[\sigma v_t+\sigma^2 s_t-\frac{\sigma^2}{2}\nabla \log p_t(x)-\sigma v_t-\sigma^2 s_t+\frac{\sigma^2}{2}\nabla \log q_t(x)\right]^\top \nabla \log \frac{p_t(x)}{q_t(x)} \cdot p_t(x) dx\\
&= -\frac{\sigma^2}{2} \int \left\|\nabla \log \frac{p_t(x)}{q_t(x)}\right\|^2 p_t(x) dx \leq 0\, .
\end{align*}
\end{remark}

\section{SECOND-ORDER DYNAMICS}
\label{sec:second_order}
The motivation of this section is to design a smoother trajectory for $X_t$ while still maintaining optimality (and therefore uniqueness) of the dynamics. Suppose the reference $Q$ is given by the following augmented process with a velocity variable $V_t\in \mathbb{R}^d$, and $(X_t,V_t)$ have as the stationary distribution $\mathcal{N}(0,I)\otimes \mathcal{N}(0,I)$:  
\begin{align}
\label{eqn:ref_2nd_1}
dX_t &=  V_t\, dt\\
\label{eqn:ref_2nd_2}
dV_t &= -X_t dt-\gamma  V_t dt + \sqrt{2\gamma}\, \overrightarrow{dW_t}
\end{align}
We assume it is initialized at $Z_0:=(X_0,V_0)\sim \nu\otimes \mathcal{N}(0,I)$ independent and would like to enforce $Z_T:=(X_T,V_T)\sim \mu\otimes \mathcal{N}(0,I)$ at the terminal. In this second-order case \citep{dockhorn2022score}, to solve the optimal trajectory problem, we formulate $P$ as the controlled process
\begin{align}
\label{eqn:2ndorder-1}
&\inf_{\rho,u} \; \int_{\mathbb{R}^{2d}} \int_0^T \gamma\|u_t(Z_t)\|^2 \rho_t(Z_t)\, dt dz\\
\label{eqn:2ndorder-2}
&\text{s.t.}\; dX_t= V_t\, dt\\
\label{eqn:2ndorder-3}
&\quad \quad dV_t = -[X_t+2\gamma u_t(Z_t)] dt -\gamma  V_t dt + \sqrt{2\gamma}\, \overrightarrow{dW_t},   \; \rho_0(Z_0)=\nu\otimes \mathcal{N}(0,I), \rho_T(Z_T)=\mu\otimes \mathcal{N}(0,I)
\end{align} 
One could check via the Girsanov's theorem that by picking 
\[\bar{b}_t = \begin{bmatrix} V_t\\ -X_t-\gamma V_t\end{bmatrix}, \quad b_t = \begin{bmatrix} V_t\\ -X_t-2\gamma u_t(Z_t)-\gamma V_t\end{bmatrix}, \quad G_t = \begin{bmatrix}0 & 0\\ 0 & \sqrt{2\gamma}I\end{bmatrix}\]
since $\bar{b}_t-b_t\in \text{image}(G_t)$, the KL divergence between the two path measures $P$ and $Q$ is $\mathbb{E}[\int_0^T \gamma \|u_t(Z_t)\|^2 dt]$. This is the SB goal of 
\[\min_{P_{0}\sim\nu\otimes\mathcal{N}, P_{T}\sim \mu\otimes\mathcal{N}} D_{KL}(P\Vert Q)\,, \]
which if minimized will identify the path with the minimum control effort between the two time marginals. We note that this formulation is different from the one in \cite[Section 3]{chen2024deep}, and is closer to an underdamped version of the SB dynamics. We denote the corresponding controlled path measure \eqref{eqn:2ndorder-2}-\eqref{eqn:2ndorder-3} as $\overrightarrow{\mathbb{P}}$. 

The Radon-Nikodym likelihood ratio from Lemma~\ref{lem:RN} still applies, therefore to impose path measure consistency, we introduce a backward control $o(Z_t,t)$ for the process $\overleftarrow{\mathbb{P}}$: 
\begin{align}
dX_t &=  V_t\, dt\\
dV_t &= -[X_t-2\gamma o_t(Z_t)] dt-\gamma  V_t dt + \sqrt{2\gamma}\, \overleftarrow{dW_t}\, .
\end{align}
The KL divergence $D_{KL}(\overrightarrow{\mathbb{P}}\Vert\overleftarrow{\mathbb{P}})$ for the augmented process becomes (from now on we will take $\gamma=2$ for simplicity corresponding to critical damping) 
\begin{equation}
\label{eqn:likelihood_ratio_second_order}
\mathbb{E}_{Z\sim \overrightarrow{\mathbb{P}}}\left[\int_0^T 2\|u_t(Z_t)+o_t(Z_t)\|^2 - \nabla \cdot \hat{b}_t(Z_t)\, dt 
+\log \frac{\nu(X_0)\otimes\mathcal{N}(V_0)}{\mu(X_T)\otimes\mathcal{N}(V_T)}\right]+C 
\end{equation}
for $\gamma=2$ and 
\[\hat{b}_t = \begin{bmatrix} V_t\\ -X_t+2\gamma o_t(Z_t)-\gamma V_t\end{bmatrix}\, ,\]
with $\nabla\cdot\hat{b}_t(Z_t)=-\gamma d+2\gamma \text{div}_V (o_t(Z_t))$. We minimize over vector fields $u(Z_t,t), o(Z_t,t)$ taking gradient forms to impose time reversal consistency. Below we derive a corresponding PINN and a variance regularizer for this dynamics, first on the forward drift followed by the backward drift. 
\paragraph{Forward condition} Writing out the optimality condition for the optimization problem \eqref{eqn:2ndorder-1}, we have for a Lagrange multiplier $\lambda(Z_t,t)$, since by continuity equation the dynamics \eqref{eqn:2ndorder-2}-\eqref{eqn:2ndorder-3} can be rewritten as
\begin{align} 
\label{eqn:augment_rho}
\partial_t \rho_t(Z_t) + \nabla_{[X,V]} \cdot\left(\rho_t(Z_t)\begin{bmatrix} V_t \\-X_t-4 u_t(Z_t)-2 V_t -2 \nabla_V\log \rho_t(Z_t)\end{bmatrix}\right)= 0 \, ,
\end{align}
we seek to optimize the Lagrangian
\begin{align*}
\int_{\mathbb{R}^{2d}} \int_0^T 2\|u_t(Z_t)\|^2 \rho_t(Z_t) + \lambda(Z_t,t)\left[\partial_t \rho+\nabla_{[X,V]} \cdot\left(\rho_t(Z_t)\begin{bmatrix} V_t \\-X_t-4 u_t(Z_t)-2 V_t -2 \nabla_V\log \rho_t(Z_t)\end{bmatrix}\right)\right] \, dtdz\, .
\end{align*}
Performing integration by parts, we have
\begin{align*}
\int_{\mathbb{R}^{2d}} \int_0^T \left\{2\|u_t(Z_t)\|^2 -\partial_t \lambda (Z_t,t)-\nabla_{[X,V]} \lambda (Z_t,t)^\top \begin{bmatrix} V_t \\-X_t-4 u_t(Z_t)-2 V_t \end{bmatrix}-2\Delta_V \lambda(Z_t,t)\right\}\rho_t(Z_t) \, dtdz\, .
\end{align*}
Fixing $\rho$, the optimal control $u_t\in \mathbb{R}^d$ satisfies
\begin{equation}
\label{eqn:control}
u_t(Z_t)=-\nabla_V \lambda(Z_t,t)\, .
\end{equation}
Plugging this choice back in, the optimal $\lambda(Z_t,t)$ must satisfies the following PDE:
\begin{align}
&2\|\nabla_V \lambda(Z_t,t)\|^2-\partial_t \lambda(Z_t,t) -\nabla_X \lambda(Z_t,t)^\top V_t+\nabla_V \lambda(Z_t,t)^\top (X_t-4\nabla_V \lambda(Z_t,t)+2V_t)-2\Delta_V \lambda(Z_t,t) = 0 \nonumber \\ 
&\Rightarrow \partial_t \lambda(Z_t,t)=-2\|\nabla_V \lambda(Z_t,t)\|^2-\nabla_{[X,V]} \lambda(Z_t,t)^\top  \bar{b}_t  -2\Delta_V \lambda(Z_t,t) 
\label{eqn:second_order_pde}
\end{align}
where the $\Delta_V \lambda(Z_t,t)$ term is interpreted as $\text{Trace}(\nabla^2_V \lambda (Z_t,t))$. Turning to the variance loss, use It\^o's lemma we get that along the reference SDE \eqref{eqn:ref_2nd_1}-\eqref{eqn:ref_2nd_2},
\[d\lambda_t(Z_t)=\left[\frac{\partial \lambda_t(Z_t)}{\partial t}+\nabla_X \lambda_t(Z_t)^\top V_t-\nabla_V \lambda_t(Z_t)^\top (X_t+2 V_t)+2\Delta_V \lambda_t(Z_t)\right] dt+2\nabla_V \lambda_t(Z_t)^\top dW_t\, ,\]
which by plugging in the PDE \eqref{eqn:second_order_pde}, we end up with
\[d\lambda_t(Z_t) = -2\|\nabla_V \lambda(Z_t,t)\|^2\, dt + 2\nabla_V \lambda(Z_t,t)^\top dW_t\, .\]
It implies the regularizer on the forward control as
\begin{equation}
\label{eqn:var_second_order}
\text{Var}\left(\lambda(Z_T,T)-\lambda(Z_0,0)+\int_0^T 2\|\nabla_V \lambda(Z_t,t)\|^2dt-2\int_0^T \nabla_V \lambda(Z_t,t)^\top dW_t 
 \right).
\end{equation}
Once we have the optimal control $u(Z_t,t)$ from \eqref{eqn:control}, one can simply simulate \eqref{eqn:2ndorder-2}-\eqref{eqn:2ndorder-3} from $Z_0$ to draw samples from $\mu$. We recommend a splitting approach where one alternates between the Hamiltonian part (with reversible symplectic leapfrog denoted by flow map $\Phi_h$), the $u_t$ drift part (with Euler), and the rest OU part (with exact simulation) for SDE discretization.

\paragraph{Backward condition} Via a change of variable  
\[\log\rho_t(Z_t)=\lambda_t(Z_t)+\eta_t(Z_t),\quad u_t(Z_t)=-\nabla_V\lambda_t(Z_t)\, ,\]
now \eqref{eqn:augment_rho} and \eqref{eqn:second_order_pde} imply the PDE regularizer on the backward control $o(Z_t,t)=-\nabla_V \eta(Z_t,t)$ should read as 
\begin{equation}
\partial_t \eta (Z_t,t)=-\nabla_{[X,V]} \eta(Z_t,t)^\top \bar{b}_t-\nabla\cdot \bar{b}_t+2\|\nabla_V \eta (Z_t,t)\|^2+2\Delta_V \eta(Z_t,t)\, ,
\end{equation}
and the corresponding SDE regularizer along the reference \eqref{eqn:ref_2nd_1}-\eqref{eqn:ref_2nd_2} becomes (using $\nabla\cdot \bar{b}_t=-2d$):
\[d\eta_t(Z_t) = \left[2\|\nabla_V \eta(Z_t,t)\|^2+4\Delta_V \eta(Z_t,t) +2d\right]\, dt + 2\nabla_V \eta(Z_t,t)^\top dW_t\, .\]
Moreover the boundary condition should match the marginal densities at $t=0,T$ in the sense that time reversal imposes $u_t(Z_t)+o_t(Z_t)=-\nabla_V\log\rho_t(Z_t)$ for all $t\in [T]$.
From \eqref{eqn:augment_rho} the deterministic ODE implementation of the dynamics with given $(u_t(Z_t),o_t(Z_t))$ therefore becomes
  \begin{equation}
  \label{eqn:ode_1}
  dZ_t = \begin{bmatrix} dX_t \\ dV_t \end{bmatrix} = \begin{bmatrix} V_t \\-X_t-2 V_t-2 u_t(Z_t) +2 o_t(Z_t)\end{bmatrix} dt\, .  
  \end{equation}

\begin{remark}[Cost and evolution]
\label{rem:second_order}
We would like to note that the reference process \eqref{eqn:ref_2nd_1}-\eqref{eqn:ref_2nd_2} still gives an entropy regularized optimal transport objective with quadratic cost, since 
\begin{align}
&\arg\min_{P_0\sim \nu\otimes \mathcal{N},P_T\sim \mu\otimes \mathcal{N}}\, KL(P_{0T}\Vert Q_{0T}) \label{eqn:augmented_sb}\\
&= \arg\min_{P_0\sim \nu\otimes \mathcal{N},P_T\sim \mu\otimes \mathcal{N}}\, \mathbb{E}_{P_{0T}}[-\log Q_{T|0}]-\int P(Z_0,Z_T) \log P(Z_0,Z_T) \, dZ_0 dZ_T \nonumber\\
&=\arg\min_{P_0\sim \nu\otimes \mathcal{N},P_T\sim \mu\otimes \mathcal{N}} \, \mathbb{E}_{P_{0T}}[\underbrace{(Z_T-M_{0T}Z_0)^\top D^{-1}_{T|0}(Z_T-M_{0T}Z_0)}_{\mathcal{C}_T(Z_0,Z_T)}]-H(P_{0T})  \nonumber\, , 
\end{align}
where we used that the conditional distribution of $Z_T$ given $Z_0$ is Gaussian, which we denote with mean $M_{0T}Z_0$ and covaraince $D_{T|0}$. Using the fact that the reference is a linear SDE that admits analytical solution, assuming we initialize as
\[(X_0,V_0)\sim \mathcal{N}\left(\begin{pmatrix}\mu_X^0\\ \mu_V^0\end{pmatrix},\begin{pmatrix}\sigma_{XX}^0\cdot I & 0\\ 0 & \sigma_{VV}^0\cdot I\end{pmatrix}\right)\, ,\]
the dynamics will remain Gaussian with
\[\mathbb{E}[Z_T] = \begin{bmatrix}(T+1)\mu_X^0 + T \mu_V^0\\ -T \mu_X^0+(1-T)\mu_V^0 \end{bmatrix} e^{-T} = \begin{bmatrix}(T+1)e^{-T}\cdot I &  T e^{-T}\cdot I\\ -Te^{-T}\cdot I & (1-T)e^{-T}\cdot I\end{bmatrix}\begin{pmatrix}\mu_X^0\\ \mu_V^0\end{pmatrix}\rightarrow 0_{2d}\]
and
\[
\text{Cov}[Z_T] = \begin{bmatrix} 
\sigma_{XX}^T \cdot I & \sigma_{XV}^T\cdot I \\ \sigma_{VX}^T\cdot I & \sigma_{VV}^T\cdot I \end{bmatrix} e^{-2T} \rightarrow \begin{bmatrix}
I & 0\\
0 & I
\end{bmatrix}
\]
for 
\begin{align*}
\sigma_{XX}^T &= (T+1)^2\sigma_{XX}^0+T^2\sigma_{VV}^0+e^{2T}-1-2T-2T^2\\
\sigma_{XV}^T &= \sigma_{VX}^T =-T(T+1)\sigma_{XX}^0+T(1-T)\sigma_{VV}^0+2T^2 \\
\sigma_{VV}^T &= T^2\sigma_{XX}^0+(1-T)^2\sigma_{VV}^0-2T^2+2T+e^{2T}-1
\end{align*}
and plugging in $\mu_X^0= x_0, \mu_V^0= v_0, \sigma_{XX}^0 =\sigma_{VV}^0 = 0$ gives the desired $M_{0T}$ and $D_{T|0}$. This asymptotic behavior matches our intuition about the equilibrium. 

\end{remark}

Note that the coupling $P_{0T}(X_0,V_0,X_T,V_T)=\pi^*(X_0,X_T)\mathcal{N}(V_0)\mathcal{N}(V_T)$ for $\int \pi^*(X_0,X_T) dX_T = \nu(X_0)$, $\int \pi^*(X_0,X_T) dX_0 = \mu(X_T)$ is always valid, in which case the objective will simplify to a quadratic involving $(X_0,X_T)$ only since $\mathbb{E}_{P_{0T}}[V_0 V_T]=\mathbb{E}_{P_{0T}}[V_0 X_T]=\mathbb{E}_{P_{0T}}[X_0 V_T]=0$ and the constraints already enforce $\mathbb{E}_{P_{0T}}[\|V_T\|^2]$ and $\mathbb{E}_{P_{0T}}[\|V_0\|^2]$ to be fixed and $\mathbb{E}_{P_{0T}}[X_T V_T]=\mathbb{E}_{P_{0T}}[X_0 V_0]=0$. Similarly the entropy term $H(\cdot)$ will also only involve $(X_0,X_T)$ in this case. But generally the marginal over $(X_0,X_T)$ from solving the optimal transport SB problem \eqref{eqn:augmented_sb} will not correspond to $\pi^*(X_0,X_T)$ from \eqref{eqn:static}, as we show below.

\begin{lemma}[Structure of the joint optimal coupling]
The joint coupling from \eqref{eqn:augmented_sb} in general do not factorize over the $X$ and $V$ variables, unless $T\rightarrow \infty$, in which case the coupling over the $X$ variables will favor the independent one, i.e., $\nu(X_0)\otimes \mu(X_T)$.
\end{lemma}
\begin{proof}
The optimal regularized coupling always takes the form for some $f,g$ 
\begin{align*}
\pi^*(Z_0,Z_T) &= e^{f(Z_0)}\cdot e^{g(Z_T)}\cdot Q(Z_T|Z_0)\cdot \nu(X_0)\otimes \mathcal{N}(V_0)\otimes \mu(X_T)\otimes \mathcal{N}(V_T)\\
&= e^{f_1(X_0)+f_2(V_0)}\cdot e^{g_1(X_T)+g_2(V_T)}\cdot Q(Z_T|Z_0)\cdot \nu(X_0)\otimes  \mu(X_T)\otimes \mathcal{N}(V_0)\otimes \mathcal{N}(V_T)\\
&= [e^{f_1(X_0)+g_1(X_T)}\nu(X_0)\otimes  \mu(X_T)]\cdot [e^{f_2(V_0)+g_2(V_T)}\mathcal{N}(V_0)\otimes \mathcal{N}(V_T)]\cdot Q(Z_T|Z_0)\, .
\end{align*}
While the rest of the terms may factorize, there will be cross terms between e.g., $X_T - V_0$ and $V_T - X_0$ coming from $Q(Z_T|Z_0)$ unless $T\rightarrow \infty$. Although the cost $\mathcal{C}_T(Z_0,Z_T)$ is still quadratic, since it approaches in that limit
\[\min_{X_T\sim\mu,V_T\sim\mathcal{N}}\; \mathbb{E}[\|X_T\|^2+\|V_T\|^2]\]
using Remark~\ref{rem:second_order}, which is fixed by the constraint, the second entropy term will dominate and therefore the resulting optimal coupling will be the product distribution.
%
\end{proof}

\begin{remark}[Gaussian entropy regularized $\mathcal{W}_2$ transport]
In the more commonly studied case when the cost is simply $\|Z_T-Z_0\|^2$, and if both $\nu$ and $\mu$ are Gaussian with arbitrary covariance, using closed-form expression from \cite[Theorem 1]{janati2020entropic} we have that the regularized optimal transport plan is Gaussian over $\mathbb{R}^{2d}\times \mathbb{R}^{2d}$: 
\begin{align*}
\pi^*\sim\mathcal{N}\left(\begin{pmatrix} [\bar{\nu},0]^\top\\
[\bar{\mu},0]^\top
\end{pmatrix}, \begin{pmatrix}
A & C_T\\C_T^\top & B
\end{pmatrix}\right)
\end{align*}
for 
\begin{align*}
A = \begin{bmatrix} 
\Sigma_\nu & 0\\ 0 & I
\end{bmatrix}\quad
B = \begin{bmatrix} 
\Sigma_\mu & 0\\ 0 & I
\end{bmatrix} \quad
C_T = \frac{1}{2}A^{1/2}(4A^{1/2} BA^{1/2}+T^2 I)^{1/2} A^{-1/2}-\frac{T}{2} I\, .
\end{align*}  
We recognize 
\[\frac{1}{2}A^{1/2}(4A^{1/2} BA^{1/2}+T^2 I)^{1/2} A^{-1/2}=A^{1/2}(A^{1/2} BA^{1/2}+T^2/4\cdot I)^{1/2} A^{-1/2}=(AB+T^2/4\cdot I)^{1/2}\] 
as the unique square root of $AB+T^2/4\cdot I$ with non-negative eigenvalues when $A\succ 0$. Therefore in this case, the optimal entropy-regularized coupling does take the form of $P_{0T}^*(X_0,V_0,X_T,V_T)=\pi^*(X_0,X_T)\bar{\pi}^*(V_0,V_T)$ where $\pi^*$ and $\bar{\pi}^*$ are respectively the entropy-regularized optimal coupling between the $X$-variables and the $V$-variables (i.e., the structure of $C_T$ implies that $X$ and $V$ are uncorrelated).
%
\end{remark}


\paragraph{Normalizing constant estimator} Importance-weighted $Z$-estimators can be derived from the controlled sampling trajectory \eqref{eqn:2ndorder-2}-\eqref{eqn:2ndorder-3} using similar ideas as those in Lemma~\ref{lem:discretization} by leveraging the likelihood ratio \eqref{eqn:likelihood_ratio_second_order}. One can check that the discretized $Z$ estimator can be written as 
\begin{align}
&\frac{1}{N}\sum_{i=1}^N \exp\Big(\log\frac{\mu(X_K)\otimes \mathcal{N}(V_K)}{\nu(X_0)\otimes \mathcal{N}(V_0)}+\label{eqn:second_order_z}\\
&\sum_{k=0}^{K-1} \frac{1}{2(1-e^{-2h})} \|\hat{V}_{k+1}^i-e^{-h}V_k^i\|^2-\frac{1}{2(1-e^{-2h})}\|V_k^i-e^{-h}\{[\Phi_{-h}(Z_{k+1}^i)]_V+4h\nabla_V \eta_{k+1}(\Phi_{-h}(Z_{k+1}^i))\}\|^2 \Big) \nonumber
\end{align}
for $3$-stage splitting updates 
\[\hat{V}_{k+1}\sim\mathcal{N}(e^{-h} V_k, (1-e^{-2h}) I),\; \tilde{V}_{k+1}=\hat{V}_{k+1}+4h\cdot \nabla_V \lambda_k(X_k,\hat{V}_{k+1}),\; Z_{k+1}=(X_{k+1},V_{k+1})=\Phi_h(X_k,\bar{V}_{k+1})\] 
from $Z_0\sim \nu\otimes \mathcal{N}(0,I)$. Notice the last two parts are purely deterministic and the first two parts act solely on the variable $V$, which explains the last line in \eqref{eqn:second_order_z}.

\newpage
\section{ADDITIONAL NUMERICAL EXPERIMENTS}
\label{app:numerics}

We provide additional details on the numerical experiments in this section. Compared to MCMC, our method is gradient free and since it takes a more global perspective on the path space, it can be less susceptible to mode collapse and/or long escape time. This type of path-based sampler also comes with importance weighting correction and has normalizing constant estimator built in as an output. 

\subsection{Benchmark and Metrics}
The following $4$ targets are considered with different priors $\nu$: 
\begin{itemize}
\item 2D standard Gaussian: $\mathcal{N}(x; 0, I)$ with prior $\nu(x) = \mathcal{N}(x; 0, 2I)$
\item Funnel: $x_0 \sim \mathcal{N}(0, 3^2 ), x_1 | x_0 \sim \mathcal{N}(0, \exp(x_0))$ with prior $\nu(x) = \mathcal{N}(x; 0, 2I)$
\item Gaussian mixture model with $9$ modes: $\frac{1}{9}\sum_{i=1}^9 \mathcal{N}(x;\mu_i, I) \; \textrm{where} \; \{\mu_i\}_{i=1}^9=\{-5,0,5\}\times \{-5,0,5\}$ with prior $\nu(x) = \mathcal{N}(x; 0, 3.5^2I)$
\item Double well: $\mu(x) \propto -(x_0^2 - 2)^2 - (x_1^2 - 2)^2$ with prior $\nu(x) = \mathcal{N}(0, 2I)$
\end{itemize}
%
For each of the benchmarks and $4$ losses, we plot the marginals and report the following: 
\begin{itemize}
\item Absolute error in mean and relative error in standard deviation compared to the ground truth using importance-weighted Monte-Carlo estimates \eqref{eqn:iw_statistics} 
\item The log normalizing constant $\log Z$ estimator (c.f. Lemma \ref{lem:discretization}):
\[\log\left(\frac{1}{n}\sum_{i=1}^n \frac{\mu(x_{K+1}^i)}{\nu(x_0^i)}\exp\left[\sum_{k=0}^K \frac{1}{2}\|Z_k^i\|^2-\sum_{k=0}^K \frac{1}{2\sigma^2 h}\|x_k^i-x_{k+1}^i-\sigma^2 h \nabla \psi(x_{k+1}^i,(k+1)h)\|^2\right]\right)\]
using the final trajectory $\{x_k^i\}$ from \eqref{eqn:em-discretize}. 
\end{itemize}

\paragraph{Importance weighting}
For the very last sampling SDE, we simulate \eqref{eqn:em-discretize} with the latest $\nabla \phi$, but re-weight the $n$ samples $\{x_{K+1}^i\}_{i=1}^n$ each with individual (un-normalized) weight 
\begin{equation}
\label{eqn:discretized_w}
w^\phi(x_{K+1}^i)=\frac{\mu(x_{K+1}^i)}{\nu(x_0^i)}\exp\left(h\left[\sum_{k=0}^K\frac{\sigma^2}{2}\Delta \phi(x_k^i,kh)-\frac{\sigma^2}{2}\Delta \psi(x_k^i,kh)\right]\right)
\end{equation}
before taking their average, i.e., 
\begin{equation}
\label{eqn:iw_statistics}
\hat{\mathbb{E}}_{p_\text{target}}[g] = \frac{\sum_{i=1}^n g(x_{K+1}^i)w^\phi(x_{K+1}^i)}{\sum_{i=1}^n w^\phi(x_{K+1}^i)}
\end{equation}
for a summary statistics $g:\mathbb{R}^d\rightarrow \mathbb{R}$ we are interested in. This estimator follows from Proposition~\ref{prop:importance_sample} and is used for post-processing with a potentially suboptimal control $\nabla \phi$, caused by e.g., approximation or optimization error from neural network training. 

\subsection{Result}
For the experiments, we parameterize $\phi, \psi:\mathbb{R}^d \times [0, c] \rightarrow \mathbb{R}$ as residual feed-forward neural networks where $c$ is a hyperparameter.
Each network first maps the input vector to a hidden state with a linear transformation, which is then propagated through several residual layers that maintain hidden state dimensionality.
Each residual layer receives as input a state-time pair $(x, t) \in \mathbb{R}^{d+1}$ and outputs $\sigma(W (x, t) + b) + x$, where $\sigma$ is the ReLU activation function.
The number of hidden layers and their sizes vary by target distribution.
Adam optimizer was used to train the models with $\beta_1 = 0.9, \beta_2 = 0.999$, and weight decay $0.01$, where batches of trajectories are used for several steps of gradient updates in each epoch, before regenerating the $n$ trajectories and estimating the objective for the next round of updates on the NN parameters. Across all experiments, we initialize $\phi(x_0^i,0)\approx \log \mu (x_0^i)$. For fair comparison, the training process is stopped when the loss stops noticeably decreasing. 





For a comparison on the computation speed of the four regularizers: for the standard normal target it took 624 seconds to train with PINN, 108 seconds with the Variance regularizer, 106 seconds with TD, and 85 seconds with Separate Control. Generating trajectories requires much less time than evaluating the loss and its gradient. Computing each loss requires processing all $K$ states across $n$ trajectories. Combined with the number of training epochs and the number of updates per batch, we estimate the processing time per trajectory state as approximately
\[
    \frac{\texttt{training\_time}}{\texttt{epochs} \times \texttt{trajectories} \times \texttt{updates\_per\_batch} \times K}\,. 
\]
The stated processing time is therefore roughly $6.2 \cdot 10^{-3}\si{\second}$ for PINN, $3.6 \cdot 10^{-7}\si{\second}$ for Variance, $3.5 \cdot 10^{-7}\si{\second}$ for TD, and $2.8 \cdot 10^{-7}\si{\second}$ for Separate Control. The latter three are comparable and much faster than PINN whose Laplacian computation adds an order of magnitude processing time. This efficiency comparison is also independent of the target distribution.

For reproducibility, the anonymous Github repository can be found at the following link:

\begin{center}
\url{https://anonymous.4open.science/r/diffusion-sampler-E4F7}
\end{center}
where hyperparameters used for the experiments are listed in the corresponding notebooks.

In Table~\ref{tab:stddev-table}, Table~\ref{tab:mean-table} and Figure~\ref{plt:train_loss}, Figure~\ref{plt:GMM}, Figure~\ref{plt:others} below we present the simulation results. All experiments are run with a GeForce RTX 2080 GPU and an AMD Ryzen 9 3900X CPU.

\begin{table}[H]
\begin{center}
\begin{tabular}{@{}lcccc@{}}
\toprule
       & PINN   & Variance                 & TD     & Separate control         \\ \midrule
Standard normal & \num{0.231533} & \num{0.382497} & \num{0.240691} & \num{0.221310}  \\
Funnel & \num{0.794941} & \num{0.796652} & \num{0.794814} & \num{0.810974} \\
GMM             & \num{0.171790} & \num{0.088070}  & \num{0.134221} & \num{0.032793} \\
Double well     & \num{0.161144}  & \num{0.177658}  & \num{0.008204}  & \num{0.026290}   \\ \bottomrule
\end{tabular}
\caption{Relative error of weighted empirical standard deviation (lower is better)}
\label{tab:stddev-table}
\end{center}
\end{table}

\begin{table}[H]
\begin{center}
\begin{tabular}{@{}lcccc@{}}
\toprule
       & PINN   & Variance                 & TD     & Separate control         \\ \midrule
Standard normal & \num{0.123589} & \num{0.037648} & \num{0.013826} & \num{0.030821}  \\
Funnel & \num{0.786865} & \num{0.295754} & \num{0.255022} & \num{0.154966} \\
GMM             & \num{0.632307} & \num{0.039433}  & \num{0.288841} & \num{0.301473} \\
Double well     & \num{0.234222}  & \num{0.736263}  & \num{0.165676}  & \num{0.361951}   \\ \bottomrule
\end{tabular}
\caption{Absolute error of empirical mean (lower is better)}
\label{tab:mean-table}
\end{center}
\end{table}

\begin{figure}[H]
\begin{center}
\includegraphics[width=0.45\hsize]{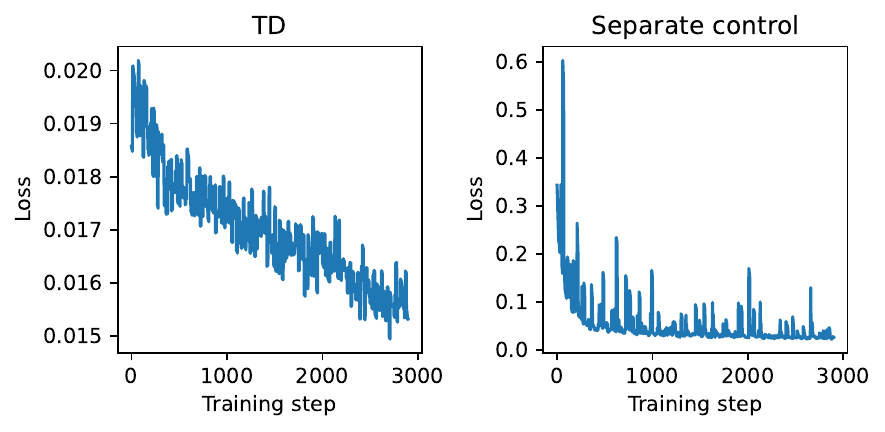}
\caption{Training loss plot for GMM with TD and Separate Control}
\label{plt:train_loss}
\end{center}
\end{figure}

\begin{figure}[H]
\begin{center}
\includegraphics[width=0.4\hsize]{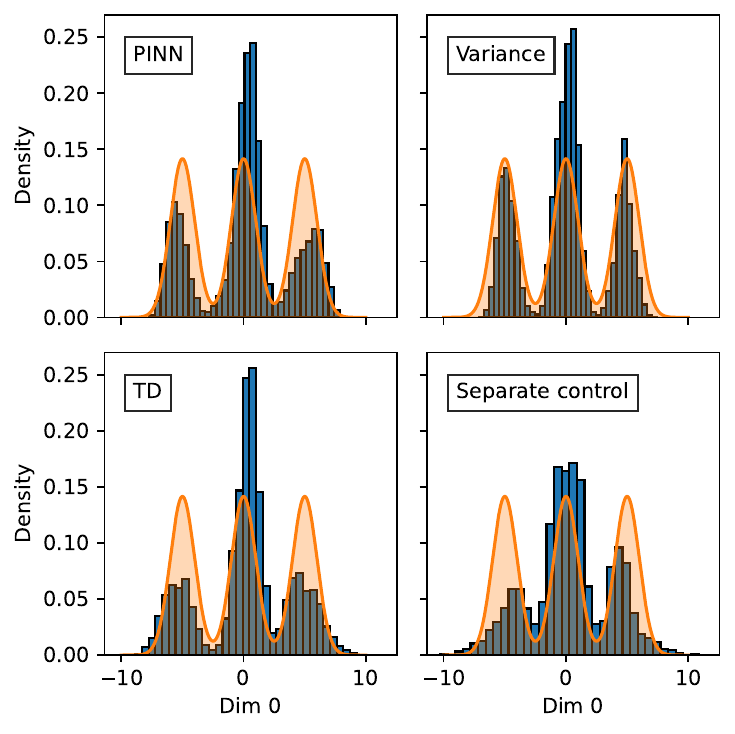}
\includegraphics[width=0.4\hsize]{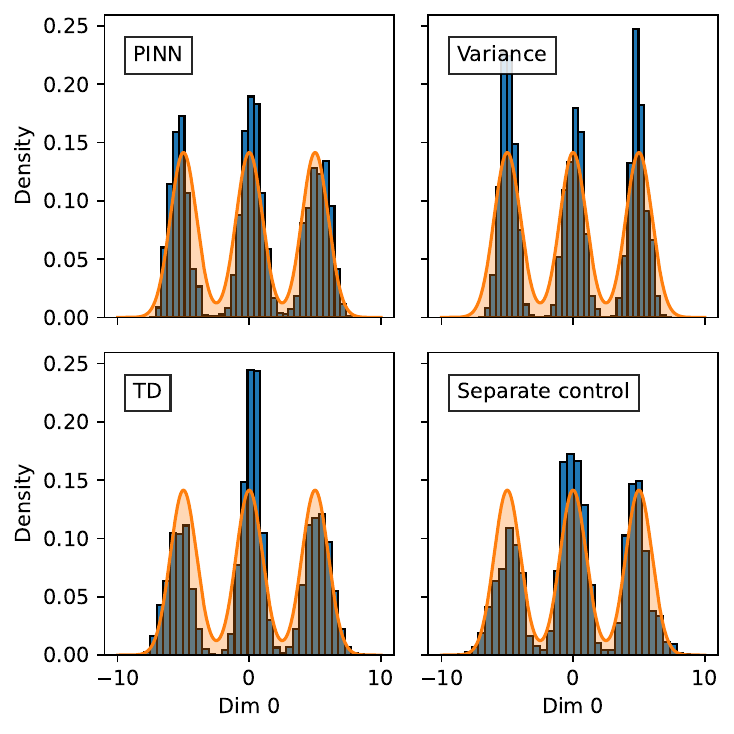}
\caption{GMM marginal before (left) and after (right) importance weighting}
\label{plt:GMM}
\end{center}
\end{figure}

\begin{figure}[H]
\begin{center}
\includegraphics[width=0.4\hsize]{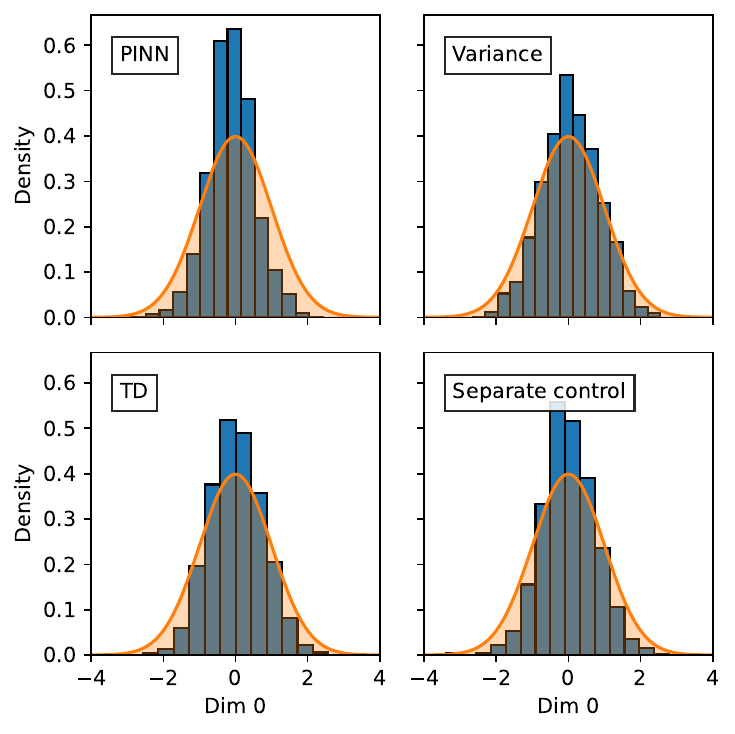}
\includegraphics[width=0.4\hsize]{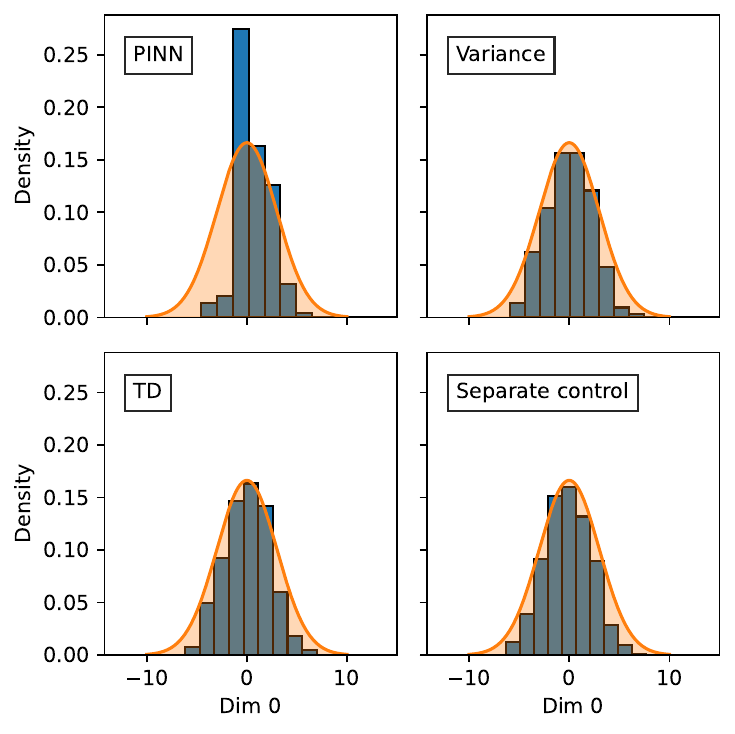}
\caption{Standard normal marginal (left) and Funnel marginal (right) with importance weighting}
\label{plt:others}
\end{center}
\end{figure}

\subsection{Application in Optimal Transport and Stochastic Control}


In contrast to \citep{chen2021likelihood}, we are guaranteed optimal control in the sense of SB, beyond just terminal marginal constraint matching. In this regard, the two works target different purposes (i.e., trajectories), and we illustrate these optimal transport and control benefits in Table \ref{tab:new_app}.

\begin{figure}[H]
  \centering
            \includegraphics[width=0.45\hsize]{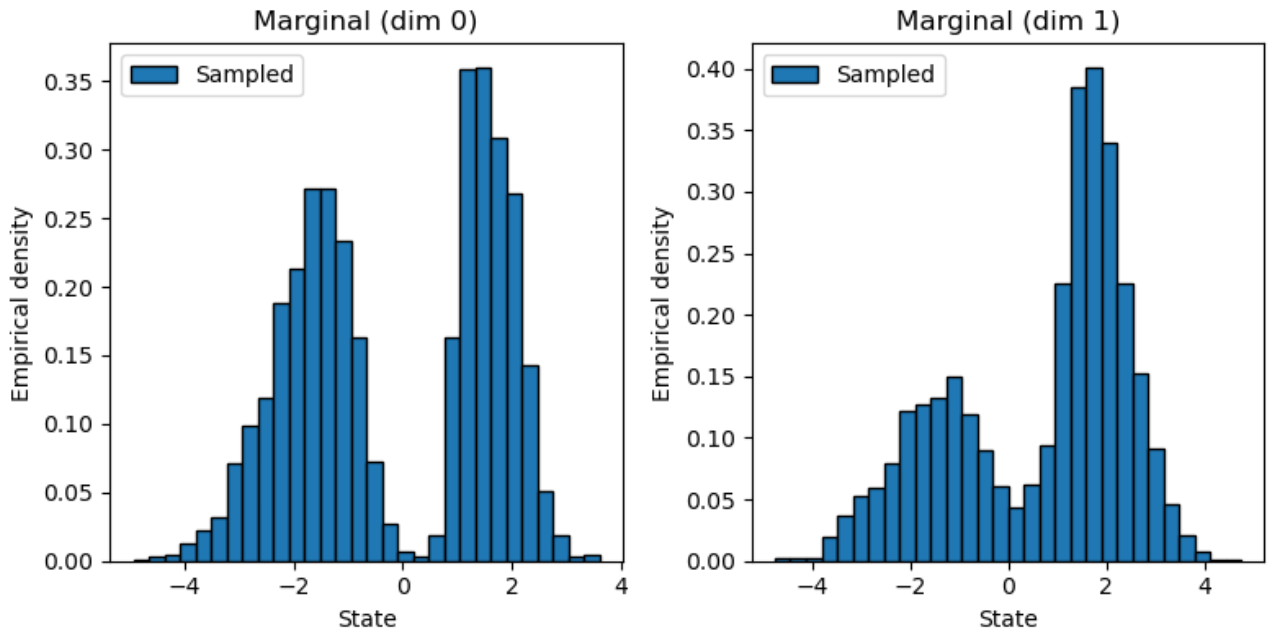}
            \hspace{0.2cm}
          \includegraphics[width=0.45\hsize]{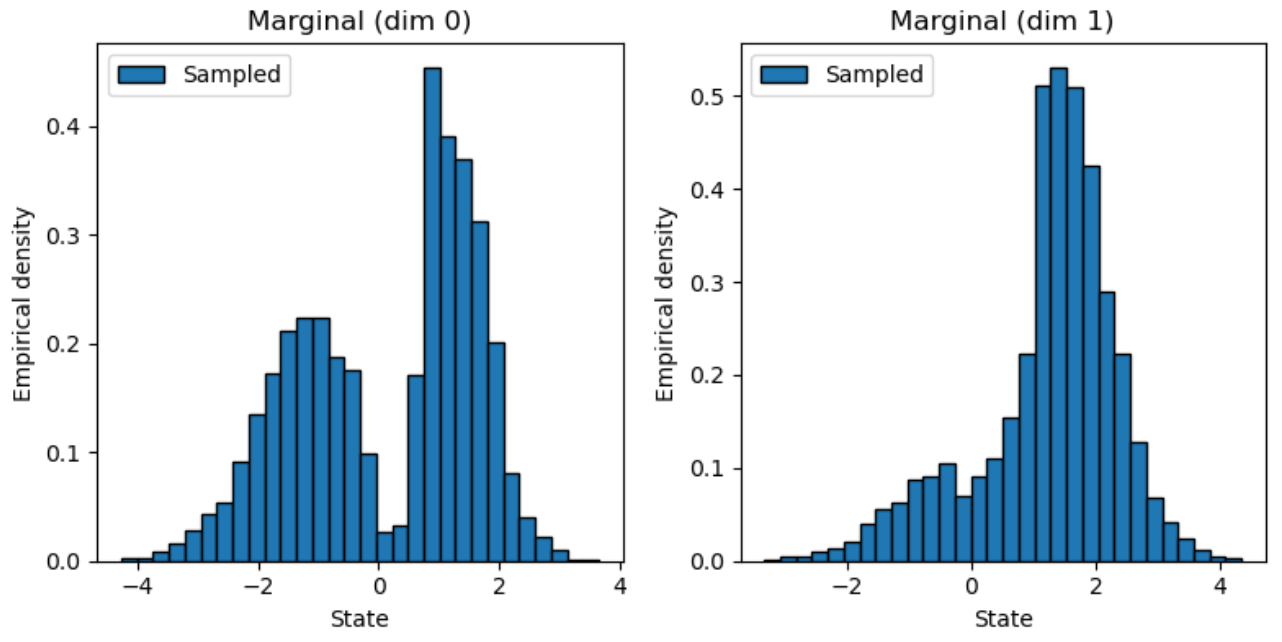}
  \caption{Double-well experiment: marginal distribution (non importance-weighted) from our variance-regularized loss \eqref{var:loss} on the left vs. method in \citep{chen2021likelihood} without regularization on the right. Absolute error in mean: 0.31 (ours) vs. 0.35 \citep{chen2021likelihood}.}
\end{figure}

\begin{table}[H]
  \centering
  \begin{tabular}{ccc}
    \toprule
    Algorithm    & Entropy-regularized $\mathcal{W}_2$ distance     & Control Cost $\mathbb{E}_{X\sim \overrightarrow{\mathbb{P}}^{\nu,\sigma^2 \nabla \phi}}[\int_0^T \frac{\sigma^2}{2}\|\nabla \phi_t(X_t)\|^2 \, dt]$ \\
    \midrule
    Ours &  1.94 &    1.62  \\
    Method of \citep{chen2021likelihood}   & 2.50   &  2.31    \\
    \bottomrule
  \end{tabular}
    \caption{Comparison of \citep{chen2021likelihood} vs. our SB method \eqref{var:loss} on double-well experiment: both are two-parameter losses but ours aim at identifying the \emph{unique optimal control} when setting the regularization parameter $\lambda>0$. We observe that our method preserves marginals with a better bridge interpolation and less control energy spent ($T=1$ is picked in all experiments).}
    \label{tab:new_app}
\end{table}



\begin{table}[H]
  \centering
  \begin{tabular}{ccc}
    \toprule
    Dimension (Method)     & Entropy-regularized $\mathcal{W}_2$ distance      & Absolute error in mean estimator \\
    \midrule
    2 (ours) & 1.94  &  0.31  \\
    2 (DDS \citep{vargas2022denoising})    &  2.28     &  0.39 \\
    \bottomrule
  \end{tabular}
    \caption{Comparison of eqn (9) in DDS \citep{vargas2022denoising} as one-parameter loss vs. our SB two-parameter loss \eqref{var:loss}. We observe (1) better trajectory property in the sense of optimal transport; (2) one requires slightly larger network size for method \citep{vargas2022denoising} to give satisfying results.}
\end{table}


\vfill


\end{document}